\documentclass[journal]{IEEEtran}

\usepackage{times}

\usepackage[numbers]{natbib}
\usepackage{multicol}
\usepackage[bookmarks=true]{hyperref}
\usepackage{mathrsfs}
\usepackage{amsmath}
\usepackage{amssymb}
\usepackage{amsfonts}
\usepackage{amsthm}
\usepackage[nameinlink,capitalise]{cleveref}
\usepackage{placeins}
\usepackage{diagbox}
\newif\ifworkversion
\workversionfalse
\usepackage{booktabs}
\usepackage{enumerate}
\usepackage{cancel}
\usepackage{graphicx}
\usepackage{algorithm}
\usepackage[noend]{algpseudocode}
\usepackage{caption}
\usepackage{lettrine}
\usepackage{balance}
\usepackage{xcolor}
\usepackage{soul}

\newcounter{sideremark}

\usepackage{amssymb}
\newcommand\E[1]{\mathbb{E}\left[\,#1\,\right]}
\newcommand{\neighbor}{\mathcal{N}}

\renewcommand{\Pr}[1]{\mathbb{P}\left[\,#1\,\right]}

\newtheorem{theorem}{Theorem}
\newtheorem{lemma}{Lemma}
\newtheorem{definition}{Definition}
\newtheorem{observation}{Observation}
\newtheorem{problem}{Problem}
\newtheorem*{assumptions}{Assumptions}
\newtheorem{assumption}{Assumption}

\renewcommand{\epsilon}{\varepsilon}

\definecolor{darkred}{rgb}{0.5,0,0}
\definecolor{lightblue}{rgb}{0,0.4,0.8}
\definecolor{darkgreen}{rgb}{0,0.5,0}
\usepackage{hyperref}
\hypersetup{colorlinks=true, linkcolor=lightblue, urlcolor=lightblue, citecolor=darkgreen}
\usepackage[subtle]{savetrees}

\begin{document}
\bstctlcite{IEEEexample:BSTcontrol}
%
\title{
Crowd Vetting: Rejecting Adversaries via Collaboration \\ with Application to Multi-Robot Flocking}
%
%
%

\author{Frederik Mallmann-Trenn, Matthew Cavorsi, and Stephanie Gil
\thanks{Frederik Mallmann-Trenn is with the Department
of Informatics, King's College London}
\thanks{M. Cavorsi and S. Gil are with the School of Engineering and Applied Sciences at Harvard University as of July, 2020 (previously at Arizona State University).}
\thanks{Manuscript submitted July, 2020.}}

%

\maketitle

\begin{abstract}
We characterize the advantage of using a robot's neighborhood to find and eliminate adversarial robots in the presence of a Sybil attack. We show that by leveraging the opinions of its neighbors on the trustworthiness of transmitted data, robots can detect adversaries with high probability. We characterize a number of communication rounds required to achieve this result to be a function of the communication quality and the proportion of legitimate to malicious robots. This result enables increased resiliency of many multi-robot algorithms. Because our results are finite time and not asymptotic, they are particularly well-suited for problems with a time critical nature. We develop two algorithms, \emph{FindSpoofedRobots} that determines trusted neighbors with high probability, and \emph{FindResilientAdjacencyMatrix} that enables distributed computation of graph properties in an adversarial setting. We apply our methods to a flocking problem where a team of robots must track a moving target in the presence of adversarial robots. We show that by using our algorithms, the team of robots are able to maintain tracking ability of the dynamic target.
\end{abstract}


%
\IEEEpeerreviewmaketitle

\section{Introduction}
%
%
%
%
\lettrine{I}{n this} paper we study the problem of computing graph properties, such as trust among neighbors and the adjacency matrix, in the presence of adversaries with application to multi-robot flocking and tracking. Robotic aerial vehicles can assist in target tracking \cite{zhou2018resilient}, search and rescue in emergency situations \cite{lewis2019developing}, and surveillance of cities or airspace for security management \cite{1678135, FAA}. Large flocks of unmanned aerial vehicles can potentially survey an airspace and keep track of other bodies in the space such as commercial drones whose presence may be unknown to the security agencies. This is gaining increasing attention as more drones are being used by industries, hobbyists, and potentially others, in a largely unregulated way. This expands the need for management of the airspace to ensure all aerial entities can coexist safely \cite{8424544, 6815901, 664154}. Thus algorithms for cooperatively tracking dynamic, and potentially uncooperative vehicles are of particular interest in this case which motivates the current paper.

\begin{figure}[b!]
    \centering
    \includegraphics[scale=0.09]{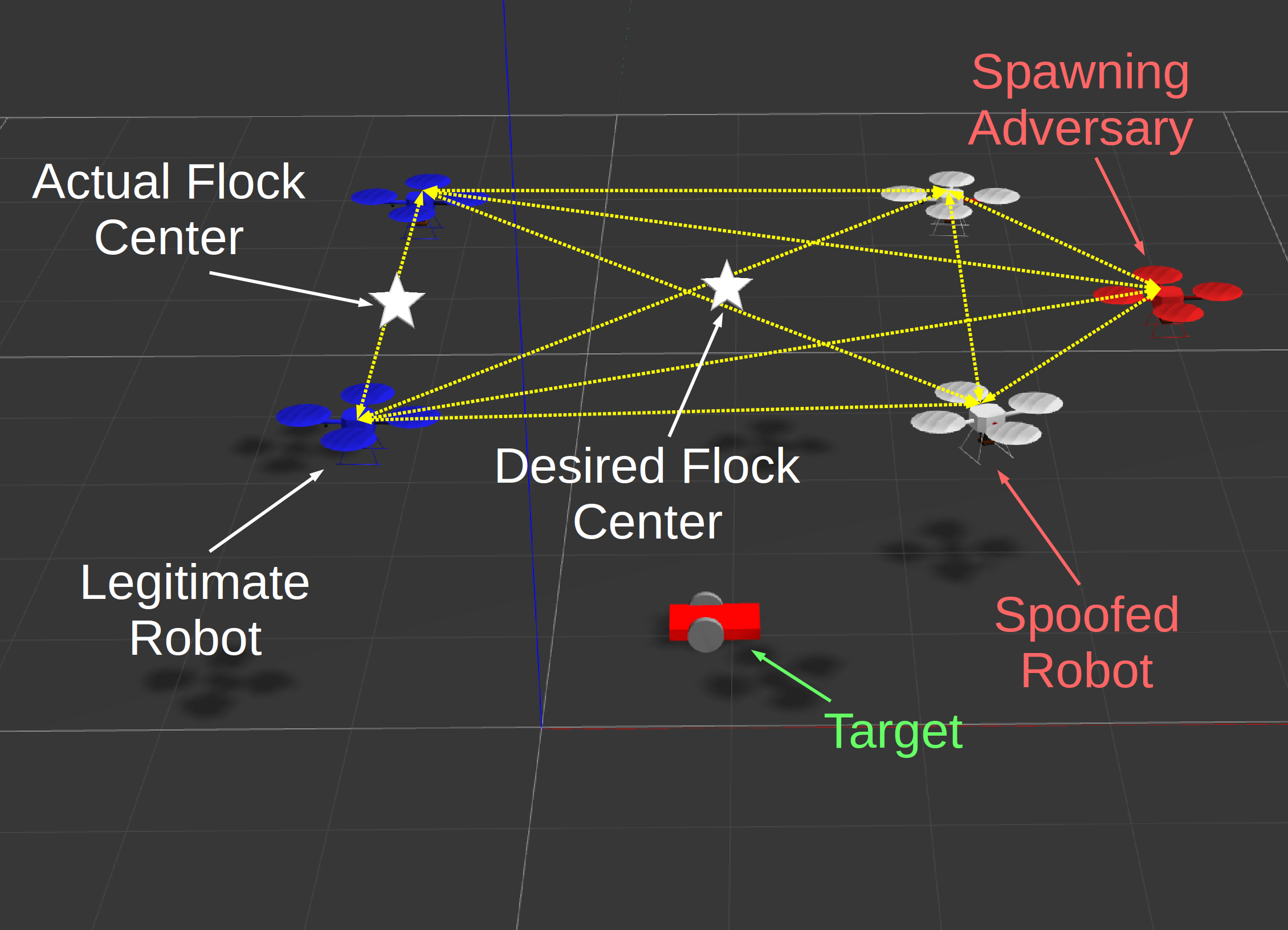}
    \caption{Depiction of a multi-robot system flocking and tracking a target in the presence of a Sybil Attack. The system contains legitimate robots, spawning adversaries, and spoofed robots. As long as the set of malicious robots is unknown, the flock is not able to accurately track the target. Communication links between robots are shown in yellow.}
    \label{fig:robots}
\end{figure}

Although multi-robot cooperation shows great promise in these domains, this cooperation is vulnerable to untrusted data when making decisions. In distributed robot teams, all members make informed decisions according to the information they have gathered from the environment and from other team members. When all the information gathered is accurate, this strategy works well, but misinformation in the system can quickly disrupt the team \cite{10.5555/3306127.3331683}.

Many works \cite{WMSR, Bullo, Sundaram}, have looked at ways to separate good information from bad information in order to regain system functionality in the face of malicious or corrupted data. These approaches however, often require underlying knowledge of the network to ensure proper dissemination of good information to counteract the negative effects of misinformation. With some underlying knowledge of the network readily available, less resilient networks can even be driven towards more resilient topologies \cite{7822915, SGcontroller}. However, in distributed systems, underlying network properties are not easily accessible and must be estimated. This estimation process is ironically very vulnerable to the same misinformation from adversaries that the aforementioned works intend to protect against. Thus, the problem of estimating network properties in the presence of adversarial influence, and as a result attaining true resilience in distributed multi-robot coordination, is a very challenging problem.

Ideally, it would be possible to understand which robots are transmitting trustworthy data without requiring all robots to be cooperative. In this way each robot can understand which subset of their neighborhood is actually trustworthy, and could therefore compute accurate local graph properties that would then be shared only with their trusted neighbors in order to construct a global understanding of graph properties such as true adjacency matrices.

For the problem of assessing the trustworthiness of data, many recent works look to the possibility of leveraging the \emph{cyberphysical} nature of multi-robot systems (such as robotic systems). If a system has access to physical information such as the locations of other robots \cite{180296}, physical sensor data such as cameras and lidars~\cite{Candell_PhysicsBasedDetection,dataVeracityCPS_entropyMethods}, or watermarking or other data extraction from the communication signals themselves~\cite{Sinopoli_Watermark,pappasSecretChannelCodes,AURO}, then this physical data can be combined with the transmitted (or cyber) data to provide more avenues for verification of the data that cannot be easily manipulated by an adversary~\cite{modelingDependability_CPS,sastryCPS_survivability,trustandRobotsSycara,trustEstimationRobots,robotTrustSchwager,spoofResilientCoordinationusingFingerprints}.  For example, previous work shows that physical information in the network can be used to detect trustworthiness in the presence of certain attacks such as a Sybil Attack \cite{AURO, GilLCSS, doi:10.1177/0278364914567793} where an adversary ``spoofs", or spawns, another robot into the system to gain a disproportionate influence on the group decision \cite{1307346}.

Our previous work~\cite{AURO,GilLCSS} builds the framework for extracting uniqueness information from wireless communication signals, allowing robots to form their own opinions of the trustworthiness of other robots through direct observation. However this previous work does not leverage the opinion of the robot's neighbors. Prior work in opinion dynamics shows that the opinion of a neighborhood can be influential in reaching a common value. In social networks, individuals with different views can eventually come to an agreement after exchanging information \cite{Acemoglu2011}. This can be a powerful tool for converging on \emph{a global understanding} of the true underlying graph structure of the network, i.e. who is trustworthy and who is not. This is the goal of the current paper.

In this work we utilize a combination of observations over wireless signals and social learning to provably converge quickly on the correct trust values (with high probability) for all robots in a team in the face of an attack. All robots will develop their own opinion about their trust of the network by observing the messages sent over the network. They then compare their opinions with that of their neighbors to reach a consensus regarding the truth on which robots can be trusted. By leveraging their neighbors opinions, each robot effectively increases the number of observations of the network available to them and through cross-verification can eliminate those messages originating from malicious robots with high probability. We show that a global consensus, where all robots agree on the trustworthiness of their neighborhoods, is possible within a finite number of communication rounds.  We also derive the probability that the converged opinion of trust is accurate.
We demonstrate our findings in a flocking scenario with a team of aerial robots tracking a dynamic object of interest, such as an aerial or ground vehicle, or another robot. Here, malicious robots feed misinformation into the system to disrupt the flock and create an escape corridor for the target. The team must be reactive enough to realize that there is a malicious influence, and resilient enough to discover the sources of that influence in order to reject them before failing their tracking objective. 

\subsection{Contributions}
We focus on the problem of coordination in the presence of a Sybil Attack where malicious robots can transmit information over the network under unique IDs (i.e. \emph{spawn} robots in the network) in order to disproportionately influence the coordination outcome. We refer to these spawned transmissions as ``spoofed'' robots and we refer to cooperating robots as \emph{legitimate}. Extending from the findings in \cite{AURO}, we assume the ability to detect spoofed and spawning robots by observing their communication signatures.
Specifically, we model the communication signature as a stochastic variable $X_j^t\in[0,1]$, where a value close to $1$ indicates a truthful communication from robot $j$ at time-step $t$, and a value close to $0$ indicates a fraudulent transmission (see \cref{sec:threat}).
Our algorithm, \emph{FindSpoofedRobots}, has three key-parts. 

First, all robots broadcast a few messages. 
Based on the communication signatures $(X_j^t(i))$, each robot $i$, establishes an interim trust vector -- whom of the other robots it is inclined to trust. 

Second, each robot shares its interim trust vectors (whom they trust) with its neighbors. 

Finally, each robot $i$ establishes a final trust vector by doing the following.
For each of its neighbors, robot $i$ uses the interim trust vectors from part 2, in order to determine whether $i$ trusts neighbor $j$.
This is done by taking the majority vote among all neighboring robots $k$ that $i$ trusts and that are also neighbors of $j$.
Our results are as follows
\begin{enumerate}
\item
We prove bounds (upper and lower bounds: \cref{thm:one} and \cref{lem:lowermaj}) on the required number of rounds, where a round is the $1$-hop communication time for all robots to synchronously communicate with their neighbors. We complement our theoretical results with extensive simulation results.
Compared to previous work \cite{AURO},
our algorithm is in many cases significantly more efficient by leveraging the opinions of the neighborhoods. In particular, for some cases where the ratio between
spoofed robots to legitimate robots is fixed, say $10$ to $1$, 
the number of rounds we need is a small constant independent on the number of robots. In contrast, we show that approaches that do not consult with their neighborhoods (such as \cite{AURO}), i.e. other robots' trust vectors, require a number of rounds that depends on the total number of robots $n$.
%
\item
We then show how our algorithm can be used as a black-box to 
render any multi-robot algorithm resilient against detectable adversaries.
In particular, we show how it can be applied to computing the global adjacency matrix as well as to target tracking.
Our bounds on required communication rounds are crucial in the highly dynamic test scenario of tracking a moving target.
\item Finally, we provide extensive theoretical analysis and simulation studies to support all of our claims in the paper and provide additional intuition.

\end{enumerate}

\section{Related Work}

Multi-robot systems that perform distributed decision making are very vulnerable to attack and hard to secure \cite{10.1145/2500423.2500444, 6870484}. An adversary can attack a distributed system to try to cause it to fail its coordination goal. This paper considers adversaries to be any robot who does not cooperate with the same update rule as the rest of the team or feeds misinformation to other robots.

A common approach employed by resilient coordination methods is to use certain connectivity or topology requirements in the network to achieve convergence in the presence of misinformation or a malicious attack.  There exists substantial literature that deals with adding resilience to adversaries for consensus and flocking systems~\cite{Bullo, 7822915,7820067,10.1145/2185505.2185507}. For example, in \cite{Bullo} it is shown that most adversarial attacks can be detected by a robot as long as the connectivity of $W$ is $2m+1$. The authors in \cite{WMSR} present the Weighted-Mean Subsequence Reduced (W-MSR) algorithm to provide resilience to malicious robots by rejecting outliers from the consensus system values. They prove that the system can reach consensus using the W-MSR algorithm with up to $m$ adversaries, as long as the network is $(m+1,m+1)$-robust. Given a nonempty subset of robots, the set is said to be ($a_1,a_2$)-reachable if there are at least $a_2$ nodes in the set, each of which has at least $a_1$ neighbors outside of the set (see \cite{WMSR} for a more in-depth discussion on reachable sets and robust networks). In \cite{Sundaram} it is shown that a consensus system can be more generally described as each robot performing a distributed function calculation of the system initial values $\frac{1}{n}\sum_{j = 1}^n x_j[0]$. The authors proceed to provide resilience to up to $m$ adversaries by determining the system initial values using the observability matrix of each robot, assuming the robot has at least $2m + 1$ internally vertex-disjoint paths to each robot outside of its neighborhood.

Each of these existing adversary defense strategies require that the network topology have a certain level of connectivity. How network topology can impact the resiliency of the system is analyzed in more detail in \cite{7822915,7820067,1272911}. However, for robots to determine the level of connectivity of the network, such as the adjacency matrix or algebraic connectivity of a graph, they must rely on distributed information passed from their neighbors which could itself be manipulated or corrupted~\cite{Sundaram}. Robots currently have no way of determining their level of resilience to adversaries while also discerning between good and bad information that they receive from the other robots. Therefore the problem of distributed estimation of graph properties in an adversarial setting is very challenging and no viable solutions to protect against these adversaries have been presented in the literature to the best of our knowledge.

Works such as \cite{7822915} and \cite{SGcontroller} design supergradient controllers to drive a network to a more connected and resilient state, but these too rely on optimizing graph properties, like the algebraic connectivity, which must be determined in a distributed fashion.  Therefore, these methods are also inherently susceptible to malicious influence.

\emph{In this paper we study the problem of constructing a global understanding of true underlying graph properties in the presence of adversarial (non-cooperative) agents and misinformation.}

One way to protect against the damage of misinformation, that we will employ in this paper, is by understanding trust for certain robots in the network \cite{6491281}. An example of how this can be achieved is by sensing anomalies in the information received by malicious robots \cite{6315661}. Because the information is anomalous, it is possible to flag the anomalous data and reject it. However, smart adversaries can inject ``fake" good information, that looks believable to other legitimate robots, but are still forged to negatively affect the system. An example of this would be to inject information that does not appear anomalous to gain false trust in the network \cite{10.1145/2185505.2185507, 6032057, 5399524, 8250942}.

In most systems, all robots communicate some private value to one another. This value could be a sensor measurement, a vote, etc \cite{10.1145/2185505.2185507}. Without the presence of adversaries, it is reasonable that all robots will communicate the value to each other and an agreement will eventually be reached. However, an adversary can communicate incorrect values to try to sway the consensus. The adversary would most likely communicate a value much different than what the others are streaming to cause a disagreement. The W-MSR algorithm \cite{WMSR} as well as others  \cite{4543196,6481629,Pierson13adaptiveinter-robot} allow a robot to attain resilient consensus on a value by comparing its value with the other values it receives, and rejecting outliers.
These methods have proven to be very successful in reaching consensus in the presence of malicious robots, but could be susceptible to more powerful adversaries. For example, if an adversary has information on what the other robots are communicating, it can tailor its messages and send values that wouldn't be ignored by the algorithms \cite{10.1145/2185505.2185507, 8250942, booker2018effects}. Also, in the case that the true connectivity of a network can not be ascertained due to the presence of non-cooperative robots, such algorithms are difficult to execute \cite{Sundaram}.

Another way of attacking a system is through a Sybil Attack, where a single malicious robot can fake multiple identities to gain a higher influence on the network \cite{1307346,origByz,douceur2002sybil,lamport}. This attack is easy to implement and potentially detrimental to multi-robot coordination \cite{douceur2002sybil}. This type of attack will be a main focus of this work. One way that has been shown to identify these adversarial robots in a network is to analyze the transmitted message signatures themselves~\cite{AURO,GilLCSS}. In \cite{10.1145/2500423.2500444}, the authors divulge the potential of using communication channels to tell when there is a problem in the network. SecureArray \cite{10.1145/2500423.2500444} utilizes multi-antenna access points to better determine the origin of a signal since the angle of arrival of the signal can be used to tell if a signal is a threat. However, this method is not a good solution for mobile robot systems due to the bulky multi-antennas needed. Inspired by this idea, \cite{AURO} and \cite{GilLCSS} developed consensus algorithms for rejecting malicious robots based on their WiFi communication angles of arrival. Synthetic Aperture Radar (SAR) \cite{doi:10.1177/0278364914567793, fitch2012synthetic} can be used to emulate a multi-antenna array by utilizing small robotic motions, such as spinning in-place, to gather information on the wireless signal from multiple angles \cite{AURO}. SAR has been used for radar positioning \cite{4132698} and indoor localization \cite{10.1145/2639108.2639142}. \cite{AURO} uses SAR to gather a client fingerprint, measure the variance of peak locations using their Signal-to-Noise Ratio, and obtain a confidence metric about the integrity of the signal. We refer to this as \emph{observations of the trustworthiness of other robots} which can be used to form an ``opinion'' about the legitimacy of other robots operating in the network.

The current paper will take this idea a step further and utilize the opinions of trusted neighbors to fortify the rejection of malicious robots and reach a consensus on trust in the network more quickly. The idea of asking for help from neighbors is inspired by opinion dynamics in social networks. Studies in opinion dynamics in social networks have shown that social learning can help lead a group to agreement even among individuals starting with different views \cite{Acemoglu2011, hegselmann2002opinion}. Assuming individuals can develop similar opinions about the legitimacy of the robots in the system, the group can then effectively aggregate information and reject the incorrect beliefs \cite{Acemoglu2011}. 

We study applications of adversary rejection for a multi-robot flocking scenario with the goal of tracking a dynamic object. Many different methods for controlling a swarm of robots have been considered in the literature. This paper utilizes flocking-based swarm control similar to \cite{8424838, 8991431, 1205192, 1605401, 10.1145/37401.37406} that was inspired by the behavior of natural swarms like flocks of birds or schools of fish. In flocking control each robot looks to satisfy three main objectives: collision avoidance, velocity matching, and flock centering. Through satisfying these goals, the flock naturally stays uniformly spaced, but still dense, moves as a unit, and has the ability to follow references which is good for object tracking.

\section{Problem Formulation}\label{sec:problem}

Consider a multi-robot coordination problem where $n$ robots must arrange themselves around a target object to track it such as in \cref{fig:robots}. Let $P := \{ p_1[t],...,p_n[t] \} \in \mathbb{R}^3$ denote the positions of each robot in the system at time-step $t$, and $\text{VEL} := \{ \text{vel}_1[t],...,\text{vel}_n[t] \} \in \mathbb{R}^3$ denote the velocities of each robot in the system at time-step $t$, for $n$ robots.  We assume that robots have a finite (arbitrary) maximum velocity of $\text{vel}_\text{max}$. The robots in the system communicate over a network, which is represented by an undirected graph, $\mathbb{G}=(\mathbb{V},\mathbb{E})$,
where $\mathbb{V}=\{1,\hdots,n\}$ denotes the set of node indices for the robots and $\mathbb{E}\subset\mathbb{V}\times\mathbb{V}$ 
denotes the set of undirected edges. We use $\{i,j\}$ to denote 
the edge connecting robots $i$ and $j$. The graph and set of edges is given.
The set of \emph{neighbors} of robot $i$ is denoted by $\neighbor_i= \{j\in\mathbb{V}\mid \{i,j\}\in\mathbb{E}\} \cup \{ i\}$.
Note that by this definition if $i$ is a neighbor of $j$, then $j$ is also a neighbor of $i$. The neighborhood of every robot $i$ is also self-inclusive, i.e., $i \in \neighbor_i$.

We are interested in the problem where an unknown subset of the robots is malicious.  In other words, the set $\mathbb{V}$ will have the subsets $L \subset \mathbb{V}$ of \emph{legitimate robots} and $M \subset \mathbb{V}$ of \emph{malicious robots}. A \emph{legitimate robot} is any robot that fully cooperates with the other robots, with $l = |L|$, and $M$ is the set of malicious robots, with $m = |M|$. We consider three types of malicious robots:

\begin{definition}[Spawning Adversary] \label{def:spawning_adversaries}
A robot is considered a \emph{spawning adversary}, $A_{spawn} \subset M$, if it is spawning other robots, i.e., broadcasting messages into the system under several unique identities.
\end{definition}

\begin{definition}[Spoofed Robot]\label{def:spoofed}
A robot is considered a \emph{spoofed robot}, $A_{spoof} \subset M$, if it is spawned, and is influencing the system under a fake (but unique) identity.
\end{definition}

\begin{definition}[Hidden Adversary] \label{def:hidden_adversaries}
A robot is considered a \emph{hidden adversary}, $A_{hid} \subset M$, if it does not spawn any spoofed robots. Since they do not spawn any robots, their messages appear to be trustworthy using our model of observations over the network (\cref{assumption:expectation}). However, they can relay incorrect or random trust vectors as a way to inject more noise into the system and try to sway the trust of others away from the truth.
\end{definition}

We say that the robots $S= A_{spoof} \cup A_{spawn}$ are \emph{detectable}. These robots can be detected through the physical characteristics of their signals (see \cref{sec:threat} for Threat Detection Model). Define $s=|S|=|A_{spawn}|+|A_{spoof}|$.

Assume the entire sets $P$ and $V$ are known, but the subsets $M$ and $L$ are unknown. We assume that the subsets $M$ and $L$ do not change throughout the rounds of execution of the algorithm.
We assume a broadcast model where messages are broadcast to other robots and we divide time into \emph{rounds}.

\begin{definition}[Round] \label{def:round}
A (synchronous) \emph{round} is the time needed to relay a message from all robots $i\in\mathbb{V}$ to their one-hop neighbors, $\mathcal{N}_i$, in the system.
\end{definition}

We use the following notation for  legitimate robot $i$. 
Let $l_{i,j}$ be the cardinality of the set $\neighbor_i \cap
\neighbor_j \cap L$ such that  $l_{i,j} = |\neighbor_i \cap
\neighbor_j \cap L|$, or, the number of common neighbors between robots $i$ and $j$ that are legitimate. Let $s_{i,j}$ be the cardinality of the set $\neighbor_i \cap \neighbor_j\cap S$ such that $s_{i,j} = |\neighbor_i \cap \neighbor_j\cap S|$, or, the number of common neighbors between robots $i$ and $j$ that are either spawned or spoofed. And let $h_{i,j}$ be the cardinality of the set $\neighbor_i \cap
\neighbor_j \cap A_{hid}$ such that $h_{i,j}=|\neighbor_i \cap
\neighbor_j \cap A_{hid}|$, or, the number of common neighbors between robots $i$ and $j$  that are hidden adversaries. Let $\tau_{i,j}$ be the gap between legitimate and hidden adversaries between robots $i$ and $j$, i.e., $\tau_{i,j}= l_{i,j}-h_{i,j}$. These values characterize the number of \emph{common} neighbors that are malicious or legitimate in the network and will be important quantities for understanding the effectiveness of leveraging neighbors trust values to converge on a global understanding of trust (see \cref{prob:trust}). Along these lines, we introduce a new concept of connectivity, $\tau$-triangular connectivity, that will be critical in the development and analysis of our results.
\begin{definition}[$\tau$-triangular]
A graph is called \emph{$\tau$-triangular} if for every legitimate robot $i \in L$ and every $j \in \neighbor_{i}$ we have that $\tau_{i,j}=l_{i,j}-h_{i,j} \geq \tau$, or equivalently, $\tau = \min_{i,j}\{ \tau_{i,j}\}$. If there are no hidden adversaries, this simplifies to $l_{i,j} \geq \tau$ meaning that $i\in L$ has for every neighbor $j$ at least $\tau$ robots in its neighborhood that are connected to both $i$ and $j$. 
Note that without adversaries
any graph in which all legitimate robots form a connected graph, $\tau \geq 2$,
since the definition of $l_{i,j}$ includes self-loops. 
\end{definition}
This concept of connectivity captures how well-connected the \emph{neighborhood} of each robot is and in particular quantifies the number of \emph{common} neighbors between robots. The definition depends on whether hidden adversaries are present. \cref{fig:minTau} demonstrates this $\tau$-triangular concept for a subgraph around two robots. 
\begin{figure}[h!]
    \centering
    \includegraphics[width = 0.45\textwidth]{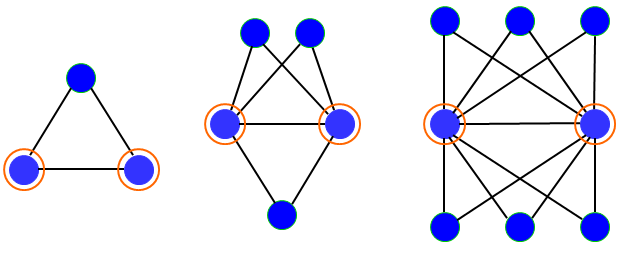}
    \caption{Consider the blue circles to be legitimate robots in a network, with the ones circled in orange to be the robot pair $\{i,j\}$. The left-most network shows a $\tau_{i,j}$ of 3 since robots $i$ and $j$ have one common neighbor and are their own common neighbors. The middle network shows a $\tau_{i,j} = 5$ and the right-most network shows a $\tau_{i,j} = 8$.}
    \label{fig:minTau}
\end{figure}

This paper does not require a high level of $\tau$-connectivity, but a high $\tau$-connectivity does capture the additional advantage that could be attained by leveraging neighbors observations in addition to each robot's direct observations of trust. Thus many of the results derived in this paper will be parameterized in terms of $\tau$.  We do require a mild assumption on connectivity however, which simply states that the graph over legitimate robots remains connected even if malicious robots were removed from the network.

\begin{definition}[Sufficiently Connected Graph] \label{def:suff_connected}
For a graph to be considered \emph{sufficiently connected}, the legitimate robots must remain connected if all other malicious robots were to be removed from the graph. This requirement is equivalent to $\tau\geq 2$.
\end{definition}


\subsection{Model of Threat Detection}
\label{sec:threat}
We assume that every message received by a robot contains data, i.e., current position and velocity of the robot, and can be tagged with a stochastic variable $X_j^t(i) \in [0,1]$: This value, or \emph{observation over the wireless channel}, relates to the uniqueness of the message sender and thus can be used to reason probabilistically about the presence of a Sybil Attack for example (see \cref{fig:attack} for an example of a Sybil Attack on flocking and tracking). More precisely,
a value of $X_j^t(i)$ close to $1$ indicates that the message $i$ receives from $j$ at time $t$ has a high likelihood of being unique. Likewise, a value close to $0$ has a high likelihood of being spoofed. These observation values have been derived in previous work and important properties about these observations for the current work are summarized below: 
 
\begin{definition}[Observation Over the Wireless Channel $X_j^t(i)$~\cite{AURO}]\label{def:observations} 
An observation $X_j^t(i) \in [0,1]$ over the wireless channel is assumed to be available. This observation denoted  $X_j^t(i) \in [0,1]$ is a random variable,  where $X_j^t(i)=0$ indicates a high likelihood of a spoofed transmission and $X_j^t(i)=1$ indicates a high likelihood of a legitimate transmission. 
\end{definition}

\begin{assumption}
\label{assumption:expectation}
We assume that  for all $i, j, t$ the $X_j^t(i)$ have  the property that
\begin{align}\label{eq:e1} \E{X_j^t(i)= 1 ~|~  \text{$j\in L \cup A_{hid} $} } \geq 1/2+\varepsilon \end{align}
	and 
	\begin{align}\label{eq:e2} \E{X_j^t(i)= 0 ~|~  \text{$j \in S$ } } \geq 1/2+\varepsilon \end{align}
where $\varepsilon \in (0,1/2)$ is related to the signal-to-noise ratio of the wireless signal and indicates the quality of the observation. A high quality observation ($\varepsilon \rightarrow 1/2$) is more likely to result in the accurate sensing of a spoofed transmission. This result is derived in \cite{AURO}, and the bounds on its expectation used here are proven as \cite[Theorem 5.4]{AURO}.
\end{assumption}
We also make the following independence assumption on the observations $X_j^t(i)$.

\begin{assumption}[Independence of the $X_j^t(i)$]\label{ass:independence}
 We assume, that the $X_j^t(i)$ are independent for all $i, j, t$.
\end{assumption}

Using these properties of the wireless signals, robots can develop a binary opinion about the other robots. An opinion of $1$ implies trust of the robot in question, while an opinion of $0$ insists the robot may be spoofed. Once robots develop opinions about the legitimacy of the other robots in the system, they can begin to share those opinions with trusted neighbors.

\begin{observation}\label{obs1}
Consider a legitimate robot $i$. Let $l_i$ denote the number of $i$'s legitimate neighbors and let $s_i$ denote the number of spoofed and spawned neighbors.
Assume $s_i > l_i$.
Without assumption \cref{eq:e1} and \cref{eq:e2}, robot $i$ cannot differentiate between legitimate and spawning of spoofed robots. 
\end{observation}

Our objective in this paper is to allow all robots to converge on a \emph{trust vector}, or opinion on which other robots in their neighborhood they can trust.

\begin{definition}[Trust vector $\mathbf{v}_i^*$]
Every robot wants to form a \emph{trust vector} $\mathbf{v}_i^*=[v_{i,1}^*,v_{i,2}^*,\hdots,v_{i,n}^*]$ where the entries $v_{i,j}^*\in\{0,1\}$ for $j\in\neighbor_i$ 
has the following meaning:
a $1$ denotes that robot $i$ trusts robot $j$ and a $0$ denotes that $i$ does not trust $j$.
The value $v_{i,j}^*$ is left undecided, denoted as $-$, if $j \notin \neighbor_i$.
\end{definition}

\begin{definition}[Convergence of a trust vector]\label{def:convergence}
An algorithm achieves \emph{convergence of a trust vector} if every legitimate robot correctly identifies all detectable malicious robots (i.e., the robots in the sets $A_{spoof}$ and $A_{spawn}$). 
\end{definition}
Intuitively speaking, our goal is for all legitimate robots to find out whom of their neighbors are spoofed or spawning robots. We summarize this objective below as Problem~\ref{prob:trust}.
\begin{problem}[Finding the final trust vector $\mathbf{v}_i^*$]
\label{prob:trust}
 We wish to derive an algorithm, that takes as input all observations $\{X_j^t(i)\}_{t\geq 0}$ for robot $i$ and its neighbors $j\in\neighbor_i$ and returns a final trust vector $\mathbf{v}^*_i=[v^*_{i,1},v^*_{i,2},\hdots,v^*_{i,n}]$. 
We further wish to characterize:
\begin{enumerate}
    \item \textbf{Number of rounds $r^*$} - the number of communication rounds $r^*$ required to reach consensus amongst all robots $i\in L$ on the trust vector $\mathbf{v}^*_i$.
    \item \textbf{Probability of success} - if the probability of success is $1-\delta$, where success is defined as the converged $\mathbf{v}^*_i$ correctly identifying all malicious robots in the network for all $i\in L$, then we wish to characterize the relationship between $\delta$ and $r^*$. 
\end{enumerate}
\end{problem}

\subsection{Distributed Estimation of Graph Properties in an Adversarial Setting} \label{sec:flocking_consensus}

For multi-robot coordination tasks it is often desirable for the robots of a team to communicate with each other and come to an agreement or consensus. This information can be an agreed upon goal or a position or velocity among other things. To keep analysis and results general, we will say that the information the robots want to agree upon is captured as a value, $x[t]$. Generally, robots can reach agreement by using an update rule of the form:
\begin{equation}
    x[t+1] = Wx[t].
\end{equation}
The matrix $W$ is called the weight matrix and contains a nonzero value in element $W(i,j)$ if $j \in \neighbor_i$, and contains zeros elsewhere. The weight matrix is called a \emph{consensus matrix} if it is row stochastic. The properties of this weight matrix determine the ability of robots to successfully coordinate and this in turn is often related to the connectivity of the graph through the adjacency matrix or Laplacian. However, knowing the adjacency matrix of the underlying graph must be achieved through distributed estimation processes that are especially difficult to achieve in adversarial settings such as the case studied in this paper.  However, if we could use the trust values learned by leveraging each robots' neighborhood as described in \cref{prob:trust}, then we could provide additional resilience to the distributed estimation process of several graph properties including the adjacency matrix.

\begin{problem}[Distributed estimation in adversarial settings]
\label{prob:distEstimation}
We wish to derive a distributed algorithm, that takes as input the final trust vectors, $\mathbf{v}_i^*$, and returns the correct adjacency matrix for the graph with  probability $1-\delta$ in unreliable networks under our Threat Model as defined in \cref{sec:threat}. With the correct adjacency matrix, other graph properties such as the algebraic connectivity and observability matrix can be determined.
\end{problem}

\subsection{Flocking Control in Adversarial Settings}
We wish to apply our results in determining inter-robot trust to the multi-robot coordination problem of flocking and target tracking.  In this problem, consensus on the position of the target to be tracked must be reached, and successful tracking is accomplished if the robot team manages to form a distributed formation around the target and move along with it.  Note that an important aspect of this problem is its dynamic nature where any momentary breaks in the formation due to adversarial activity must be recovered quickly in order to maintain tracking of the target.

\begin{definition}[Convergence Range for Flocking] \label{def:convergence_range}
A target is considered tracked by the flock if the centroid of the legitimate robots in the flock is within a desired radius of the target. From within this desired radius, the flock can converge exponentially on the target. This range is defined as the \emph{convergence range} and is derived in \cref{sec:conv_results}, \cref{eq:exp_conv_range}.
\end{definition}

Assume that the intended center of the formation is the position of the target. Denote the position and velocity of the target to be $p_\text{cen}[t]$, and $\text{vel}_\text{cen}[t]$, respectively. All robots have access to $p_\text{cen}[t]$ and $\text{vel}_\text{cen}[t]$ for all time-steps $t \geq 0$ while the target is tracked. Then following the development from~\cite{8991431} we have a Reynold's-based flocking controller for the velocity of each robot as follows:
\begin{equation}
    u_i[t] = u_i^{ref}[t] + u_i^{avoid}[t] + u_i^{match}[t],
    \label{eq:reynolds}
\end{equation}
where
\begin{equation}
    u_i^{ref}[t] = k_{ref}(p_\text{cen}[t] - p_i[t]),
\end{equation}
\begin{equation}
    u_i^{avoid}[t] = \sum_{j\in\neighbor_i} \frac{k_{avoid}(p_i[t] - p_j[t])}{\|p_i[t] - p_j[t]\|^2},
\end{equation}
\begin{equation}
    u_i^{match}[t] = \sum_{j\in\neighbor_i} \frac{k_{match}(\text{vel}_j[t] - \text{vel}_i[t])}{\|p_i[t] - p_j[t]\|}.
\end{equation}
\begin{figure}[h!]
    \centering
    \includegraphics[scale=0.1]{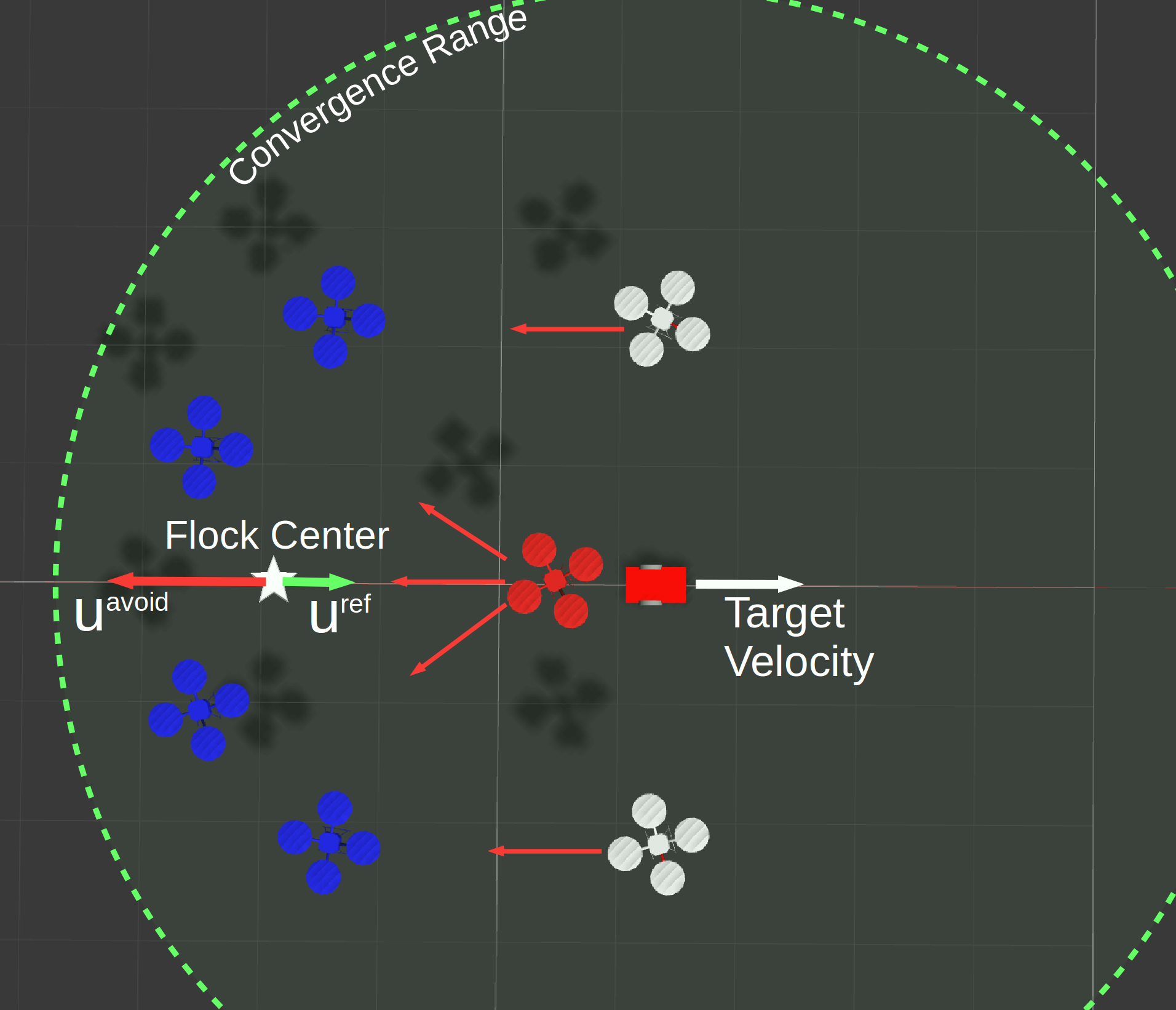}
    \caption{Adversaries attacking the flock by trying to push the legitimate robots out of the desired convergence range of the target}
    \label{fig:attack}
\end{figure}
The $u_i^{ref}$ term in the flocking controller is designed to drive each robot toward the target. The $u_i^{avoid}$ term incorporates collision avoidance between robots and encourages uniform spacing of the robots around the target center. The final term, $u_i^{match}$, integrates speed matching among robots which promotes flock velocity cohesion. The values $k_{ref}$, $k_{avoid}$, and $k_{match}$ are constants denoting the weight put on each respective term. This flocking-based controller can ensure the convergence of all cooperating robots to a formation around any stationary or moving reference point.


\begin{problem}[Flocking in Adversarial Settings]
\label{prob:flocking}
We wish to characterize a time window $\Delta$, given a particular robot formation $P$, max robot velocities $\text{vel}_\text{max}$, and target velocity $\text{vel}_\text{cen}$, where if malicious robots launch an attack at time $t=t[0]$, then robots can successfully recover tracking if malicious robots are correctly identified and rejected by time $t=t[0]+\Delta$. By comparing $\Delta$ with $r^*$ from \cref{prob:trust} we determine the conditions under which we can successfully recover tracking, i.e. remain within \emph{Convergence Range} of the target, with  probability $1-\delta$.
\end{problem}

Finally, for ease of reference we summarize key assumptions made throughout the paper.

\begin{assumptions}\label{ass:graph}
\begin{enumerate}
\item We assume a sufficiently connected graph topology over legitimate robots. Furthermore, we will assume that $\tau_{i,j} >0$, since we argue in  \cref{sec:neccessary} that when $\tau_{i,j} < 0$ reaching consensus is impossible.
\item We also assume that observations, $X_j^t(i)$, are independent for all $i,j,t$.
\item We assume that the observation quality, $\varepsilon > 0$, which means that detection of malicious robots is possible (c.f. \cref{assumption:expectation}). 
\item We will assume that all robots and spoofed identities have unique identities. We will assume that all rounds of communication happen synchronously, meaning all robots broadcast to their neighbors simultaneously.
\end{enumerate}
\end{assumptions}

A table with all notation used throughout this paper can be found at the end of the Appendix in \cref{tab:notation}.

\ifworkversion
\onecolumn
\fi

\section{Algorithm}
In this section, we describe the algorithm \emph{FindSpoofedRobots} which returns a converged trust vector for robot $i$ as described in \cref{prob:trust}. Pseudocode for \emph{FindSpoofedRobots} is shown in~\cref{FSN} and the accompanying analysis characterizing its correctness is in~\cref{sec:analysis}.

Each robot $i$ executes the same algorithm, which consists of the following two parts. In the first part (lines 1 and 2), robot $i$ computes the \emph{interim trust vector}, which indicates which robots $i$ believes to be legitimate (the $j$th entry is $1$  if $i$ trusts $j$, $0$ if it does not trust $j$ and ``$-$" if they are not connected).
In order for robot $i$ to establish for each other robot $j$, whether it will trust $j$, each robot broadcasts $r$ messages. To determine whether robot $i$ trusts robot $j$, $i$ simply adds up $X_j^t(i)$ over all rounds $r$. If this value is greater than or equal to $r/2$, then robot $i$ temporarily trusts robot $j$. The number of messages, $r$, that each robot broadcasts is determined by running another algorithm, \cref{FSN2}. The  number of rounds needed $r^*$ can be characterized as a function of the number of legitimate robots, malicious robots, hidden adversaries, and the quality of observations that have been received (c.f. \cref{thm:one}). However since these quantities are often unknown, \cref{FSN2} provides an anytime approach that runs continuously in the background, feeding \emph{FindSpoofedRobots} more rounds to broadcast each time. The idea is that the longer the robots are allowed to run the algorithms, the higher the probability that they output the correct trust vectors. \cref{thm:one} and \cref{lem:lowermaj} establish the relationship between the number of rounds $r$ completed by \cref{FSN2} and the probability $\delta$ that the returned trust vector is correct. Each time a new trust vector is generated, if it is different than the one before, the new trust vector is adapted as the best current solution, since it has the highest probability of having the correct trust. 

For a more detailed explanation of the process of determining the number of messages, $r$, that each robot broadcasts, see \cref{sec:alg2_explanation}. Each time a new trust vector is adapted as the best current solution, that trust vector can be fed into any coordination problem algorithm that utilizes trust vectors. This is done as an example with an algorithm we introduce, \emph{FindResilientAdjacencyMatrix} (\cref{alg:FRAM}), which is used for distributed estimation of graph properties in the presence of adversaries. \cref{alg:FRAM} and \cref{FSN2} are presented in \cref{sec:dist_estimation}.

In the second part of the \emph{FindSpoofedRobots} algorithm (lines 3-6), the robots share their interim trust vectors with their trusted neighbors. Based on this, every robot $i$ makes a final assessment of whether it should trust $j$, the robot in question, by looking at all other robots $k$ in its neighborhood and their trust assessments from the first part (interim trust vectors). Then, robot $i$ simply takes the majority opinion among the robots it trusted from the first part as indicated by its \emph{interim trust vector}. The result of the final assessment is summarized in the \emph{final trust vector}. From there on, each robot $i$ will upon receiving a message from a robot $j$ check the $j$th entry in the final trust vector and if robot $i$ believes that robot $j$ is not legitimate, the message is simply ignored. \cref{fig:alg_schematic} provides a visual depiction of the process run by \emph{FindSpoofedRobots} and \cref{sec:proofs} provides analysis of the performance of this algorithm both in terms of rounds $r$ and the probability $\delta$ that the returned trust vector is global (such that all robots agree on the final trust vector) and correct.

\begin{algorithm}[h!]
\caption{FindSpoofedRobots at robot $i$
\\
Input: number of rounds $r$, opinion $o$
\\
Output: trust vector $\mathbf{v}_i^*$.}\label{FSN}
\begin{algorithmic}[1]

\State Broadcast  for $r $ rounds the opinion $o$.

\State At the end of round $r$, compute the \emph{interim trust vector} $\mathbf{v}_i  \in \{0,1,- \}^n$, with \[ \mathbf{v}_i := [v_{i,1}, v_{i,2}, \dots, v_{i,n}]^\intercal \] where the $k_{th}$ entry, $v_{i,k}$, is $1$ if $i$ believes that $k$ is legitimate, that is $v_{i,k}=1$ if $\sum_{t=1}^{r} X^{t}_k(i) \geq r/2$ and $v_{i,k}=0$ if it does not trust $k$. For all robots $k$ that are not neighbors of robot $i$, the symbol `$-$' is used. Set $v_{i,i}=1$.


\State In round $r+1$ broadcast the vector $\mathbf{v}_i$ to all neighboring robots.

\State Let $\mathbf{v}:= (\mathbf{v}_1 , \mathbf{v}_2 , \dots \mathbf{v}_n )$ be a $n\times n$ matrix where the $i_{th}$ column is the vector $\mathbf{v}_i$. 

\State Evaluate for every neighboring $j$ whether $i$ believes
$j$ is legitimate and stores this in $ v_{i,j}^*$. The reevaluation is done by aggregating the information of all robots that $i$ currently believes are legitimate.

\[  v_{i,j}^* = \begin{cases}
1 & \begin{split} \text{ if } |\{ k \colon  v_{i,k}=1 \text{ and } v_{k,j}=1 \}| \\ \geq |\{ k \colon v_{i,k}=1 \text{ and } v_{k,j}=0 \}| \end{split} \\
0& \text{ otherwise}
\end{cases}.\] 

Let \[ \mathbf{v}_i^* := [v_{i,1}^*, v_{i,2}^*, \dots, v_{i,n}^*]^\intercal \]
be the corresponding \emph{final trust vector} vector.

\State Output $\mathbf{v}_i^*$
\end{algorithmic}
\end{algorithm}

\vspace{-3mm}

\subsection{Algorithm Intuition}
Recall that in the first phase, robot $i$ temporarily trusts $j$ if and only if 
$\sum_{t=1}^r X_j^t(i) \geq r/2$. This is possible due to our assumption in
\cref{eq:e1} and \cref{eq:e2} saying that in expectation $X_j^t(i)$ is strictly larger than $1/2$ for legitimate robots and strictly smaller than $1/2$ for spoofed robots.
To prove that each interim trust vector is sufficiently good, we rely on  Azuma-Hoeffding
concentration bounds (\cref{thm:Azuma}) that allows us to prove concentration for non-binary variables. 

After the first part of the algorithm is completed (Lines 1-2), we are in a setting where with high probability, each legitimate robot classifies a sufficient number of robots correctly, but any given robot is very unlikely to classify all robots correctly. This is exemplified in \cref{fig:comp_to_BM}.
However, using  strong concentration bounds for binomially distributed random variables with very small expected value (\cref{thm:bounds-binomial-distribution}), we can show that after asking all trusted robots for their opinion (Line 5), all legitimate robots find all the spawning and spoofed malicious robots with probability $1-\delta$.

\begin{figure}[h!]
    \centering
    \includegraphics[width = 0.35\textwidth]{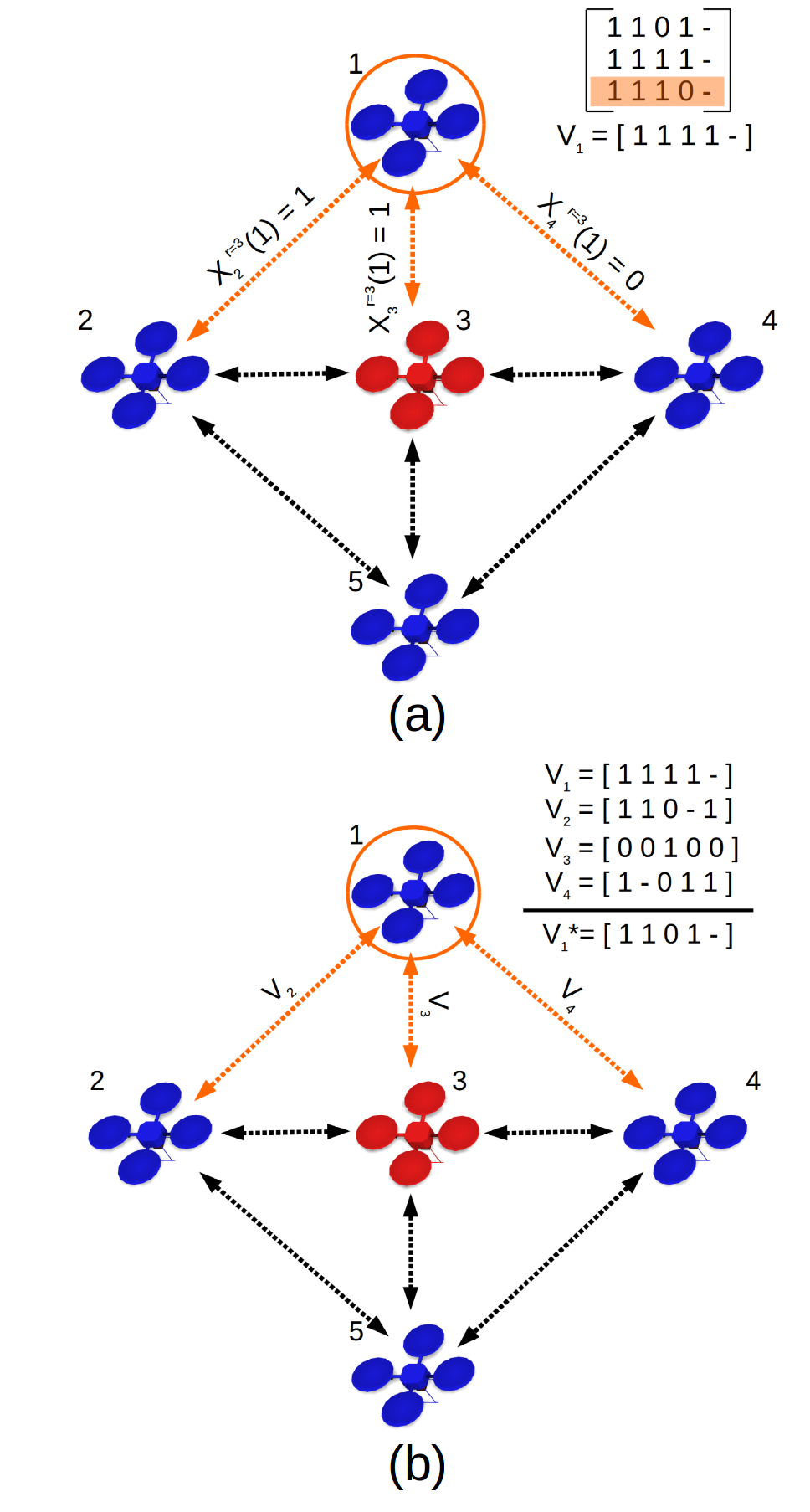}
    \caption{Visual depiction of \emph{FindSpoofedRobots} algorithm for a small network of 4 robots with 1 malicious robot that is spoofed. Part (a) shows Lines 1-2 of the algorithm where robot 1 communicates with its neighbors, robots 2, 3, and 4, for $r = 3$ rounds, with the third round of communicated data shown. The robot then forms $\mathbf{v}_1$ as in line 2 of \cref{FSN}. In (b) robot 1 receives $\mathbf{v}_2$, $\mathbf{v}_3$, and $\mathbf{v}_4$ from its trusted neighbors and determines a final trust vector, $\mathbf{v}_1^*$, as in line 5. Note that neighboring robots 2 and 4 help convince robot 1 that robot 3 is actually malicious.}
    \label{fig:alg_schematic}
\end{figure}

\section{Theoretical Analysis}
\label{sec:analysis}
In this section we prove certain desirable properties of the algorithm \emph{FindSpoofedRobots}. Namely, that 1) robots can converge on a global opinion of which other robots in the network can be trusted, 2) that this convergence happens in a finite number of rounds $r^*$ that we characterize, and 3) that the converged trust vector accurately identifies spoofed and spawning malicious robots with probability of at least $1-\delta$ where $\delta$ depends on several parameters such as the number of rounds $r$ and the quality of the observations $X_i^j(t)$.

\subsection{Trust Analysis}

Our objective in this section is to use observations $X_j^t(i)$ for each robot and its neighbors in order to arrive at a consensus on which robots' transmissions can be trusted as quickly as possible. In addition, we characterize the number of communication rounds needed to achieve this as a function of the quality $\epsilon$ of the observations, $X_j^t(i)$, and the number of legitimate and malicious robots.

Recall that $\varepsilon$ governs the probability of correctly identifying a transmission as illegitimate (see \cref{eq:e1}). 
In order for \cref{FSN} to yield the correct final trust vectors, the number of rounds, parameter $r$, needs to be larger than $r^*$, the optimal number of rounds. The value $r^*$  depends on
$n, \varepsilon, \text{ and } \tau $,
however, we stress that the algorithm does not have to know these quantities.
They are purely for analysis purposes. Instead, the algorithm uses \cref{FSN2} to probe upper bounds and will eventually find a suitable $r^*$. 

\begin{theorem} \label{thm:one}
Fix a legitimate robot $i$.
Assume $\tau > 0$. Let
$\delta >0$ and
 \[r^* :=  \left\lceil \frac{\log\left(\frac{ 2 l n}{\delta \varepsilon^2 \tau}\right)}{\tau \varepsilon^2} +\frac{\log\left(\frac{2e^{e} d_{L}}{\tau}\right)}{\varepsilon^2} \right\rceil+ 1\] rounds, where  and $e$ is the Euler constant and $d_{L}$ is the maximum degree among legitimate robots.

Then, \emph{FindSpoofedRobots} achieves that robot $i$'s trust vector is correct   w.p. at least $1-\delta$ 
for all rounds $r\geq r^*$.

\end{theorem}

In \cref{sec:neccessary} we argue that our assumptions are indeed necessary. To get a better understanding of how we improve upon the algorithm from our previous work~\cite{AURO}, which we will refer to as the \emph{Baseline algorithm}, we evaluate \emph{FindSpoofedRobots} on the example of a complete graph. For other more general graphs, see \cref{fig:with_spoofing} and \cref{fig:without_spoofing}.
\begin{definition}[Baseline algorithm \cite{AURO}]
\emph{Baseline algorithm} is the algorithm that uses direct observations over the wireless channel to determine trust of neighboring robots. It does not take advantage of trusted neighbors' opinions to improve results.
\end{definition}

\subsection{Example on the Complete Graph}

In particular, consider the following example on the complete communication graph $G$, with $l$ legitimate robots, $h=l/2$ hidden adversaries, $s=10\cdot l$ detectable adversaries and $\epsilon = 1/3$.

\FloatBarrier
\renewcommand{\arraystretch}{1.2}

\begin{table}[ht!]
    \begin{center}
    \begin{small}
    \begin{tabular}{|l|ccc|}
    \hline
    \diagbox{Required rounds}{Legitimate robots $l$} & 10 & 100 & 1000 \\
     \hline
     Baseline $r$ & 22 & 32 & 43\\
     \hline
    FindSpoofedRobots $r^*$ & 10 & 7 & 6\\
    \hline
    \end{tabular}
       \caption{Empirically obtained bound on the number of rounds $r/r^*$ that are necessary  to obtain the correct trust vectors w.p. at least $1/2$ for a fully connected graph $G$.
}
    \label{tab:examplevalues}
    \end{small}
    \end{center}
    \vskip -0.1in
\end{table}
\FloatBarrier

The required number of rounds is given in \cref{tab:examplevalues}. Note the reduction in the number of rounds $r^*$ as $l$ grows. 
This seems counter-intuitive, but can be understood by making a few key observations. 
First, note that the term $r^*$ becomes $\left(c_1\log(l)/l\right)+c_2$ for suitable constants $c_1$ and $c_2$:
to see this, note that $\tau = l/2$ and $d_L = l$. This means as $l$ grows, the first term , $\left(c_1 \log(l)/l\right)$, becomes negligible and hence the required number of rounds, $r^*$ is bounded by a constant. Note that this is in stark contrast to the Baseline, which requires at least $c \log(n)$ rounds, as we prove in \cref{lem:lowermaj}, where $c$ is a suitable constant.





Note that this example was for a fully connected graph $G$, and in general it depends on various parameters of the graph topology.

\subsection{Lower Bound For the Baseline Algorithm}

We now show that sharing trust information among each robot's neighborhood yields a tremendous advantage. To this end, we show that Baseline, or more general, \emph{any algorithm that does not communicate shared trust values}, requires a large number of rounds. We do this by showing tight bounds on the number of trust observations required (\cref{lem:lowermaj} and \cref{lem:maj}).
\begin{lemma}[Lower bound]\label{lem:lowermaj}
There exists a universal constant $c>0$ such that for any algorithm $\mathcal{A}$ that does not share observations of trust with its neighbors results in the following.
In order to obtain the correct trust vector for {\bf a given} robot w.p. at least $1-\delta$ requires at least
$ c \log(1/\delta)/\epsilon^2$ observations $X_i^t(j)$.

Moreover,
in order to obtain the correct trust vector for {\bf all} robots w.p. at least $1/2$ requires at least
$ c \log(n)/\epsilon^2$ observations.
\end{lemma}
The idea of the proof is to reduce obtaining correct trust vectors to $(\epsilon,\delta)$-testing (\cref{def:ICALPlower} and \cref{thm:ICALPlower}). The full proof can be found in \cref{sec:proofs}.

We now analyze one of the main building-blocks underlying our algorithm, \emph{FindSpoofedRobots}.
Let $\operatorname{legitimate}(j)$ be the indicator variable that is $1$ if and only if $j$ is a legitimate robot. 
\begin{lemma}[Upper bound]\label{lem:maj}
If a robot $i$ receives $r=\frac{\log(1/\delta)}{2\varepsilon^2}$ observations from another robot $j$, it will know with probability $(1-\delta)$ whether that robot $j$ is a legitimate robot by simply relying on the majority of the observations:
\[ \Pr{ v_{i,j}=\operatorname{legitimate}(j) } \geq 1-\delta.\]
\end{lemma}
The proof idea is to use the Azuma-Hoeffding inequality.
Note that we cannot use the algorithm from \cref{thm:ICALPlower} in \cref{sec:proofs}, since the domain of $X_{j}^t(i)$ is not necessarily discrete. See \cref{sec:proofs} for the proof.

Moreover, in \cref{lem:lowermaj}, we  show that, up to constants, the round complexity of \cref{lem:maj} is tight.
 
We now have all buildings blocks  to prove the main theorem, \cref{thm:one}. The proof idea is to use very strong concentration bounds (\cref{thm:bounds-binomial-distribution}) to prove that each robot is very likely to obtain the correct trust vector after $r^*$ rounds. In the second part of the proof, we show that for $r\geq r^*$, the error forms a geometric series, yielding the desired bound. See \cref{sec:proofs} for the proof.

\newcommand{\mindelta}{\delta^+}
\newcommand{\magic}{2}



\subsection{Necessity of our Assumptions}\label{sec:neccessary}
Note that our algorithm \emph{FindSpoofedRobots} can easily tolerate the case that the spawning and spoofed robots $s$ exceeds the number of legitimate robots $l$.
On the other hand, it is easy to see that without the Assumptions in \cref{eq:e1} and \cref{eq:e2}, there is no hope of finding the spoofed robots. This is because spawning and spoofed robots are now undetectable and since the  legitimate neighbors are in the minority, communication becomes futile.
Hence, robot $i$ can only guess.

Moreover,  our assumption $\tau >0$ is also necessary.
Suppose $\tau < 0$, then there exists $i\in L$ and $j\in \neighbor_i$ such that $h_{i,j}$ ($i$'s hidden neighbors) is larger than $l_{i,j}$.
In this case, $i$ has, regardless of the algorithm, no way of knowing that it can trust $j$ (all hidden robots between $i$ and $j$ may not have any other connections). Hence, $i$'s trust vector entry for $j$ can only be guessed in this case.

\subsection{Unknown Upper Bound for $r^*$}
\label{sec:alg2_explanation}
Suppose we aim to solve the consensus problem given in \cref{prob:distEstimation}.
To solve this problem, it is necessary that every legitimate robot computes the correct trust vector for its neighborhood. 
Recall that $r^*$, defined in \cref{thm:one}, is a bound on the number of rounds required for \emph{FindSpoofedRobots} to succeed with probability $1-\delta$. Using the same algorithm for more rounds $r\geq r^*$ increases the success probability further.

What can we do if we do not know the parameters  of $r^*$ (e.g., $\tau$ or $l$)?
There are two solutions. The first one is to use worst-case bounds.
\footnote{If the total number of robots, or an upper bound on it, is known, then we can show (\cref{lem:maj}), that using $r = 20 \log (n/\epsilon^2)/\epsilon^2$ achieves this. However, often this is very wasteful (e.g., \cref{tab:examplevalues}). In addition, if $n$ is not known, this approach is not possible.}
However, such bounds are not always available and the approach is very conservative.

The second solution is to use an \emph{anytime} algorithm. The idea here is that one can simply estimate an upper bound, $\hat r$, on $r^*$. This estimate is then doubled in the background (until we reach the worst-case bound or forever) and we simply execute the consensus algorithm to solve \cref{prob:distEstimation}  using the current estimate $\hat r$ and assuming it's correct. 
If that estimate $\hat r$ happened to be larger than $r^*$, then the outcome will be correct.
If it was not correct, then the current solution of the consensus algorithm may be incorrect, however, eventually we
will find the correct bound through the doubling of $\hat r$. Therefore, this is an anytime algorithm in the sense that it returns a solution at any point and the solution quality can only improve given a larger number of runs.

We show an example of such an algorithm in \cref{FSN2} applied to the problem of finding an adjacency matrix.

\ifworkversion
\twocolumn
\fi

\section{Estimating  Graph  Properties  in  Adversarial  Settings}

Estimation of graph properties in adversarial settings poses a major difficulty to resilience.  Indeed, many methods achieving resilience in multi-robot networks rely on accurate calculation of graph properties such as the Laplacian matrix, adjacency matrix, observation or control matrices, among others \cite{WMSR, Sundaram}. However, the question of \emph{how to accurately compute graph properties in adversarial settings} where robots cannot be assumed to cooperate, is largely unanswered.  In this section our objective is to derive methods for using inter-robot trust observations to estimate true underlying graph topology; i.e. selectively exclude malicious robots that could provide a false sense of security if otherwise accounted for.

Determining global properties of a network, such as the weight matrix or adjacency matrix which can show the connectivity of all robots to each other, is often necessary in order to satisfy assumptions made for resilience to adversaries in consensus systems~\cite{Sundaram,BulloCyberphysSecurity_GeometricPrinciples}. However, determining a global property in a distributed system requires collaboration and cooperation of all robots in the network, therefore making it very susceptible to adversaries. Even with global information, i.e., a fully connected graph where distributed computation is not required, computing the adjacency matrix in an unreliable network is a hard problem.  During a Sybil Attack where some robots in the network are spoofed,  simply examining each robot's neighborhood without a concept of which robots are trustworthy could give robots performing distributed estimation a false sense of security. Consider the simple example in \cref{fig:with_spoofing}, similar to~\cite{WMSR}.

\begin{figure}[h!]
    \centering
    \includegraphics[width = 0.38\textwidth]{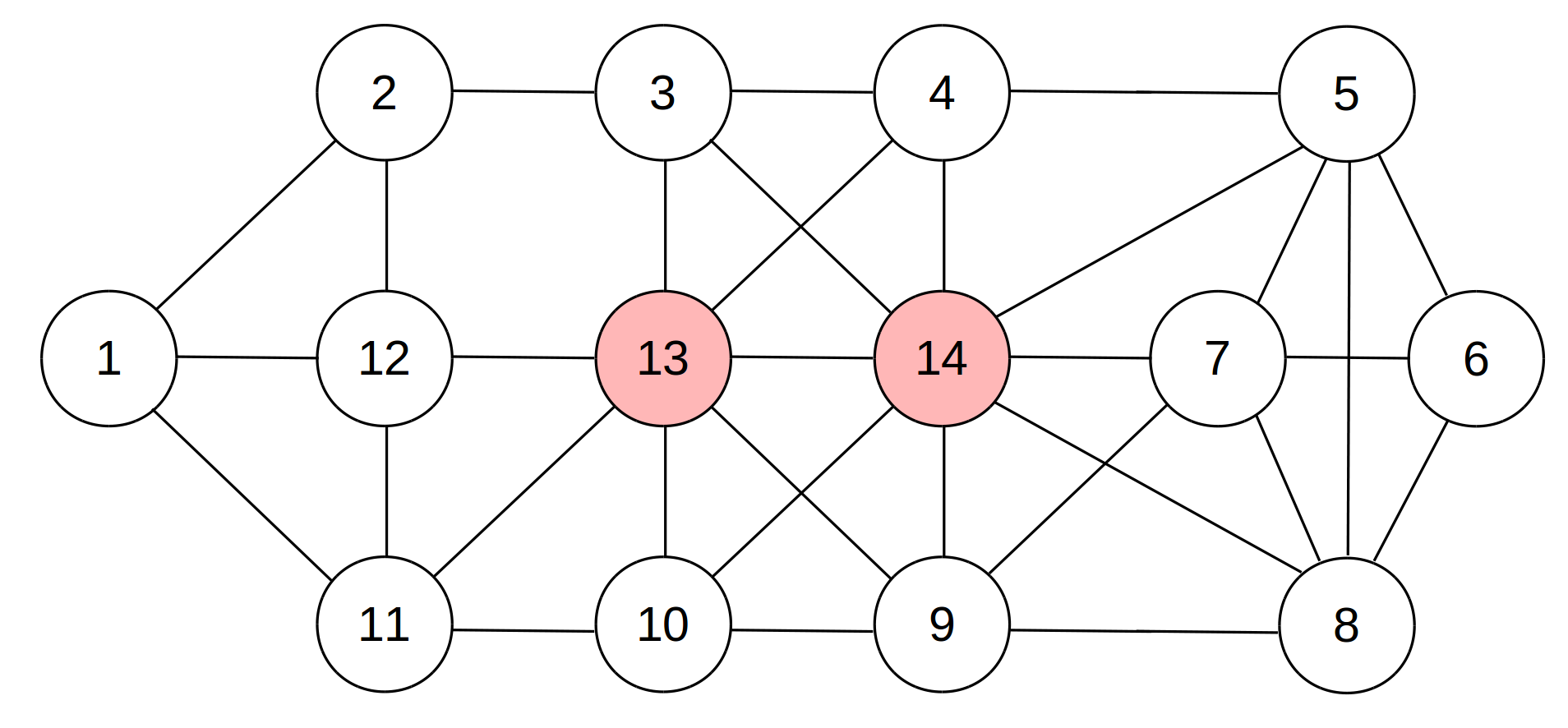}
    \caption{A 12 robot network with 2 additional spoofed robots, robots 13 and 14.}
    \label{fig:with_spoofing}
\end{figure}
\begin{figure}[h!]
    \centering
    \includegraphics[width = 0.38\textwidth]{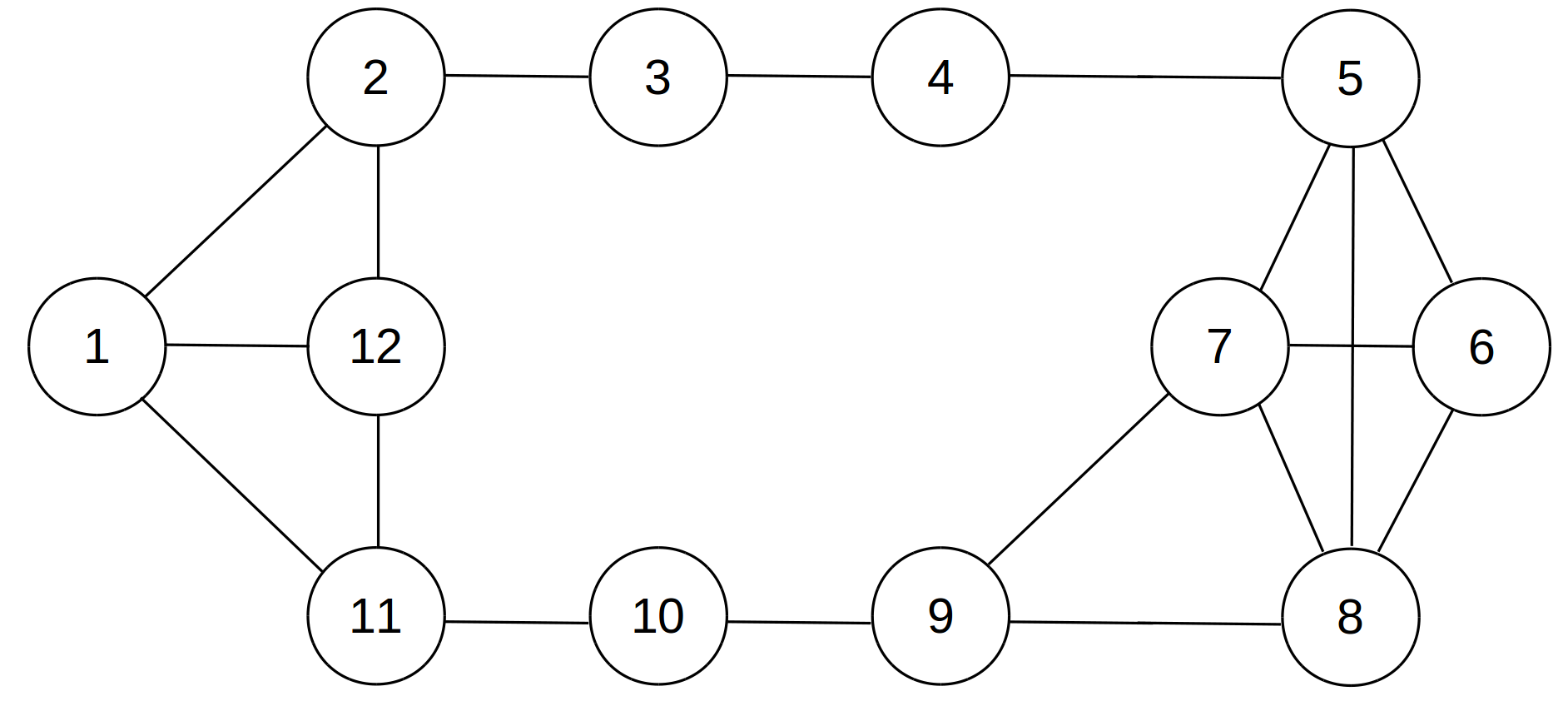}
    \caption{A 12 robot network with no spoofed robots}
    \label{fig:without_spoofing}
\end{figure}
\vspace{-3mm}
\noindent In this example, 12 robots are incorrectly perceiving their network to actually consist of 14 robots. The network is 3-connected so the algorithms in \cite{Bullo} can identify up to $m = 1$ malicious robots. Furthermore, there are three internally vertex-disjoint paths from every robot outside of $\mathcal{N}_1$ to robot 1. For example, there are three paths from robot 4 to robot 1 ($4 \rightarrow 3 \rightarrow 2 \rightarrow 1, 4 \rightarrow13 \rightarrow 12 \rightarrow 1, 4 \rightarrow 14 \rightarrow 9 \rightarrow 10 \rightarrow 11 \rightarrow 1$) all using disjoint subsets of the network. This quality means robot 1 can calculate any function of the system initial values in the presence of $m = 1$ adversary using the theory presented in \cite{Sundaram}. However, this sense of security is created by the two spoofed robots, robots 13 and 14, and this would actually fail in practice.

Consider the actual 12 robot network, with the two spoofed robots removed. This network is not ($2,2$)-robust, and therefore cannot actually handle a single adversary following the definition from~\cite{WMSR} for example. In fact, applying some resilience methodologies from the literature such as the W-MSR algorithm from~\cite{WMSR} would result in robots converging to several different values (i.e. divergence of a consensus algorithm). Indeed, there are only two internally vertex-disjoint paths from some robots outside of $\mathcal{N}_1$ to robot 1, so that robot 1 cannot calculate any function of the system initial values using existing theory from~\cite{Sundaram} for example. \cref{fig:spoofed_conn} shows how a spoofing (Sybil Attack) on the system can easily result in a false sense of security where the algebraic connectivity of a graph is assumed to be higher by the presence of malicious robots in the network.

A summary of the false security provided by the Sybil attack with respect to the mentioned works, named here as, W-MSR \cite{WMSR}, Identification \cite{Bullo}, and Distributed Function Calculation (DFC) \cite{Sundaram}, is shown in \cref{tab:false_security}, where the first row shows the number of malicious robots that are predicted to be tolerated for each resiliency algorithm, by perceiving the network in \cref{fig:with_spoofing}, and the second row shows the number of malicious robots each algorithm can actually tolerate, given the fact that the true network is what is shown in \cref{fig:without_spoofing}.
\begin{table}[h!]
\centering
\caption{Number of Adversaries Each Algorithm is Resilient to}
\begin{tabular}{c|c|c|c|c}
      & W-MSR \cite{WMSR} & Identification \cite{Bullo} & DFC \cite{Sundaram} \\ \hline
     Perceived & 1 & 1 & 1 \\ Actual & 0 & 0 & 0
\end{tabular}
\label{tab:false_security}
\end{table}

To further convey the magnitude of false security a Sybil attack can create, see \cref{fig:spoofed_conn}. In this figure, robots were placed randomly into a box which created a random communication network. Each network contained a certain amount of spoofed robots, which varied from 0 to 20. The number of legitimate robots remained constant at 20. A common way to quickly test the resilience of a network is to check its algebraic connectivity \cite{7822915, SGcontroller}. The algebraic connectivity is characterized by the second smallest eigenvalue of the graph Laplacian, and increasing values correspond to networks with greater connectivity, with a maximum connectivity value of $n$, corresponding to a fully connected graph. \cref{fig:spoofed_conn} shows that the algebraic connectivity of a graph that falls victim to a Sybil attack appears to be higher than it actually is. The actual connectivity of the network is characterized by the connectivity of the network with the spoofed robots ignored. The higher connectivity determined by the robots in a spoofed network gives the robots a false sense of the resiliency of the network to adversaries.
\begin{figure}[h!]
    \centering
    \includegraphics[width = 0.48\textwidth]{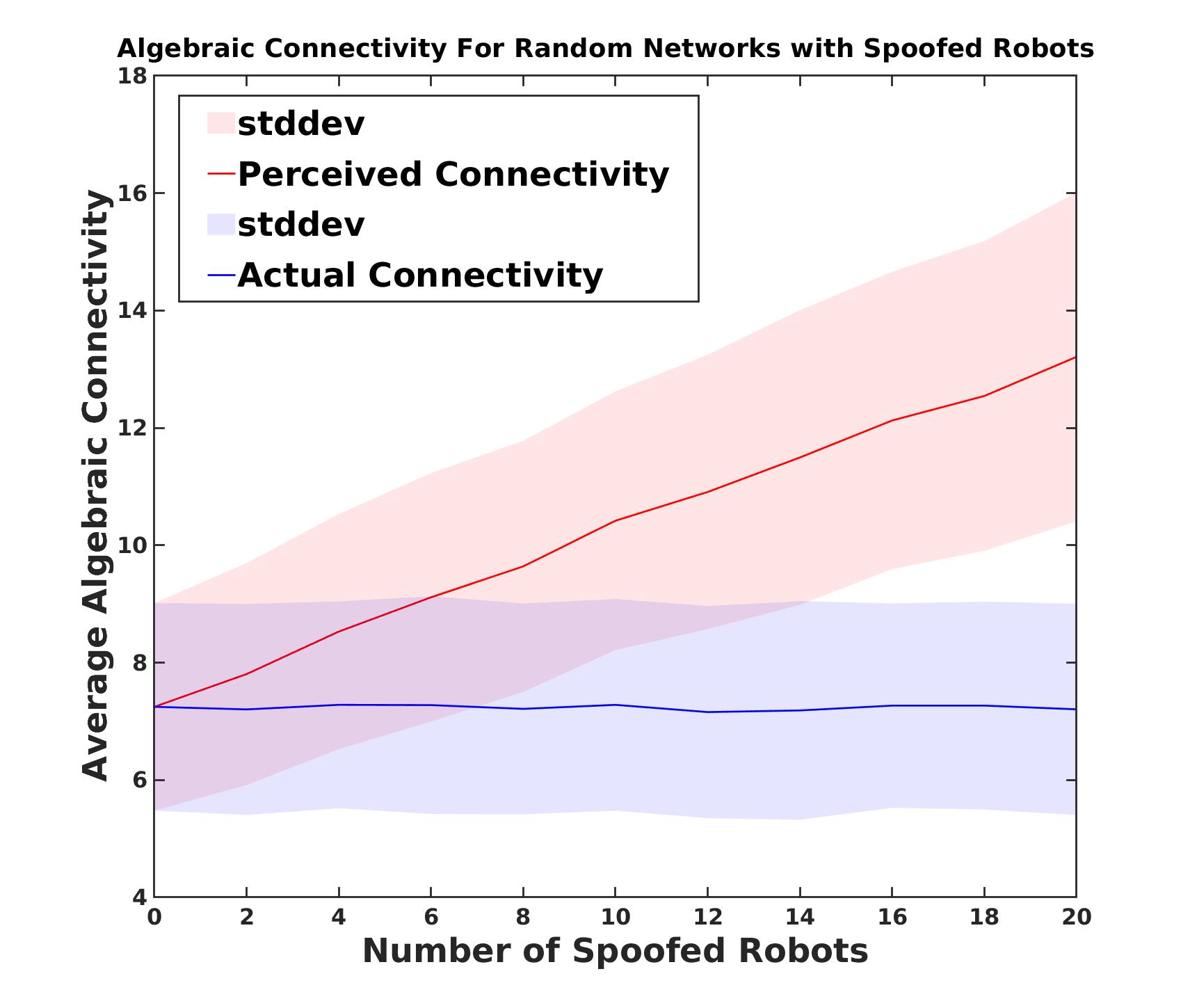}
    \caption{Random robot networks were created in the presence of a Sybil attack and the algebraic connectivity was calculated for each. Then the spoofed robots were removed and the algebraic connectivity was calculated again. The number of legitimate robots was held constant at 20 while the number of spoofed robots varied from 0 to 20. 1000 simulations were run for each number of spoofed robots and the average algebraic connectivity is shown.} 
    \label{fig:spoofed_conn}
\end{figure}

Sybil attacks can limit the resilience capabilities of \cite{WMSR, Bullo, Sundaram} by creating a false sense of security in the graph as was shown above. In the next subsection we use the trust vectors derived in the previous sections via \emph{FindSpoofedRobots} to recover true graph properties in the presence of a Sybil Attack with high probability.

\subsection{An Algorithm for Finding Resilient Graph Properties}
\label{sec:dist_estimation}

In this section we develop a new algorithm that allows a robot to utilize the trusted network that \emph{FindSpoofedRobots} provides in order to determine the true adjacency matrix of a graph in a distributed manner with high probability, even in the presence of malicious non-cooperating robots.

Determining global properties of a network, such as the weight matrix or adjacency matrix which can show the connectivity of all robots to each other, has been shown to be necessary in order to satisfy assumptions made for resilience to adversaries in consensus systems \cite{WMSR, Bullo, Sundaram, 7822915}. However, determining a global property in a distributed system requires collaboration and cooperation of all robots in the network, therefore making it very susceptible to adversaries. Even with global information, i.e., a fully connected graph where distributed computation is not required, computing the adjacency matrix in an unreliable network is a hard problem. Utilizing the \emph{FindSpoofedRobots} algorithm, robots can distributedly determine the correct adjacency matrix of the consensus system in the presence of a Sybil attack. This can be done using a new algorithm, \emph{FindResilientAdjacencyMatrix} (FRAM), which is shown in \cref{alg:FRAM}.
\begin{algorithm}[h!]
\caption{FindResilientAdjacencyMatrix
\\
Input: trust vector $\mathbf{v}_i^*$
\\
Output: adjacency matrix $\mathbf{V}_i^* = \mathcal{G}$, which is accurate with probability $1-\delta$.} \label{alg:FRAM}
\begin{algorithmic}[1]

\State Each robot runs \emph{FindSpoofedRobots} to determine their final trust vectors.

\State Each robot $i$ broadcasts its vector $\mathbf{v}_i^*$ to all neighboring robots.

\State Let $\mathbf{V}_i^* := (\mathbf{v}_1^* , \mathbf{v}_2^* , \dots \mathbf{v}_n^* )$ be a $n\times n$ matrix where the $k_{th}$ row is the vector $\mathbf{v}_k^*$ if $k \in \neighbor_i$ and is a row filled with ``$-$" entries otherwise, to indicate no data.

\State For each $k \in \neighbor_i$, change any ``$-$" in $\mathbf{v}_k^*$ to a $0$, since the lack of information in that index means the corresponding robot to that index is not a neighbor of robot $k$.

\While{$\mathbf{V}_i^*$ has at least 1 row of ``$-$" entries \textbf{or} $r \leq n-1$}

\State Each robot $i$ continues to broadcast $\mathbf{V}_i^*$ and compares its matrix $\mathbf{V}_i^*$ with its neighbors, replacing any row with ``$-$" entries with the corresponding true row as new information disperses through the network.

\EndWhile

\State Finally each robot $i$ returns its determined adjacency matrix $\mathbf{V}_i^* = \mathcal{G}$, which is accurate with probability $1-\delta$.
\end{algorithmic}
\end{algorithm}

\begin{algorithm}[h!]
\caption{FindSpoofedRobotsUnknown$r^*$ at robot $i$\phantom{xxxx}
\\
Input: neighborhood $\neighbor_i$, TrustVectorProblem that has to be solved.
}\label{FSN2}
\begin{algorithmic}[1]
\State $\hat r=1$
\State FindResilientAdjacencyMatrix($\mathbf{v}_i^*)$ 
 \While{$\hat r\leq 20 \log (n/\epsilon^2)/\epsilon^2 $}
\State $\hat r=2\cdot \hat r$
\State $\mathbf{v}_i \gets $ FindSpoofedRobots($\hat r,\neighbor_i$)
\If {$\mathbf{v}_i$ changed}
\State FindResilientAdjacencyMatrix($\mathbf{v}_i^*)$
\EndIf
\EndWhile
\end{algorithmic}
\end{algorithm}

The FRAM algorithm is demonstrated with a simple example. 

Consider the system in \cref{fig:estimation_ex}.
\begin{figure}[h!]
    \centering
    \includegraphics[width = 0.30\textwidth]{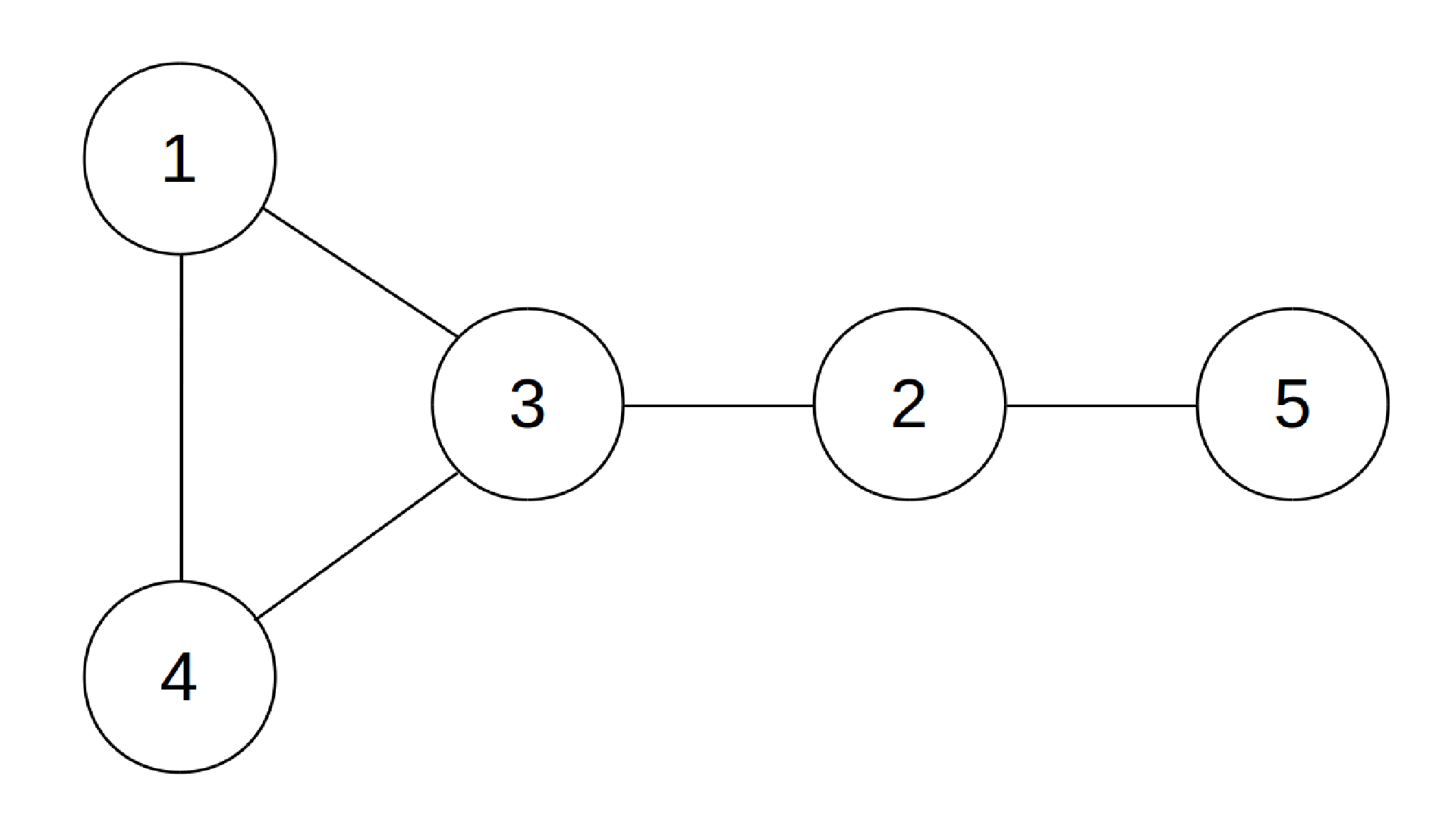}
    \caption{A simple 5 robot network to demonstrate the use of \cref{alg:FRAM}.}
    \label{fig:estimation_ex}
\end{figure}
Assume the graph to be undirected and have no adversaries with adjacency matrix
\begin{equation}
    \mathcal{G} = \begin{bmatrix} 1 & 0 & 1 & 1 & 0 \\ 0 & 1 & 1 & 0 & 1 \\ 1 & 1 & 1 & 1 & 0 \\ 1 & 0 & 1 & 1 & 0 \\ 0 & 1 & 0 & 0 & 1 \end{bmatrix}.
\end{equation}
Using our FRAM algorithm, each robot will broadcast to its neighbors and compute a trust vector
\begin{align*}
    \mathbf{v}_1^* &= \begin{bmatrix} 1 & $-$ & 1 & 1 & $-$ \end{bmatrix}, \quad \mathbf{v}_2^* = \begin{bmatrix} $-$ & 1 & 1 & $-$ & 1 \end{bmatrix}, \\ \mathbf{v}_3^* &= \begin{bmatrix} 1 & 1 & 1 & 1 & $-$ \end{bmatrix}, \quad \mathbf{v}_4^* = \begin{bmatrix} 1 & $-$ & 1 & 1 & $-$ \end{bmatrix}, \\ \mathbf{v}_5^* &= \begin{bmatrix} $-$ & 1 & $-$ & $-$ & 1 \end{bmatrix},
\end{align*}
where ``$-$" denotes no information since the robot could not communicate to all other robots. After computing the trust vector, each robot then broadcasts its vector $\mathbf{v}_i^*$ to its neighbors, allowing each robot $i$ to form the matrix $\mathbf{V}_i^*$. Only $\mathbf{V}_1^*$ is shown here for brevity:
\begin{equation}
    \mathbf{V}_1^* = \begin{bmatrix} 1 & $-$ & 1 & 1 & $-$ \\ $-$ & $-$ & $-$ & $-$ & $-$ \\ 1 & 1 & 1 & 1 & $-$ \\ 1 & $-$ & 1 & 1 & $-$ \\ $-$ & $-$ & $-$ & $-$ & $-$ \end{bmatrix}.
\end{equation}
Robot 1 still has no information about robots 2 or 5, but one can notice that rows 1, 3, and 4 of $\mathbf{V}_1^*$ are exactly rows 1, 3, and 4 of $\mathcal{G}$ (replacing the $-$'s in those rows with zeros which will happen in the next step). It is important to note that since the graph is \emph{sufficiently connected} i.e., robot 2 gathers information from robot 5 and robot 3 gathers information from robot 2, that after several iterations robot 1 would gain access to the vectors $\mathbf{v}_2^*$ and $\mathbf{v}_5^*$ and could obtain a final matrix $\mathbf{V}_1^* = \mathcal{G}$. All other robots could also follow the same process to determine their final matrix $\mathbf{V}_i^* = \mathcal{G}$.

If the system is a consensus system with a row stochastic weight matrix, then calculating the weight matrix from the adjacency matrix is straightforward. However, any general weight matrix can be calculated by extending the results from \cite{obs_estimation} for a robot estimating its observability matrix, $\mathcal{O}_i$. 

The authors in \cite{obs_estimation} outline how a robot can distributedly estimate its observability matrix, when there are no adversaries present. By utilizing our FRAM algorithm, a robot could also determine its observability matrix by using its filtered output of only trusted values and the algorithm in \cite{obs_estimation}. Assume that each robot $i$ only has immediate access to the values of its neighbors. Denote its output at time-step $t$ as $y_i[t] = C_ix[t]$ where $C_i = \begin{bmatrix} e_{j_1} & e_{j_2} & \dots \end{bmatrix}^\intercal$ for all $j \in \neighbor_i$. Then:

\begin{definition}[Observability Matrix] \label{def:observability}
The output of robot $i$ for successive time-steps can be represented as successive multiplications by the weight matrix $W$. This is known as the \emph{observability matrix} for robot $i$ and can be denoted by $\mathcal{O}_i = \begin{bmatrix} C_i & C_i W & C_i W^2 & \dots & C_i W^{n-1} \end{bmatrix}$.
\end{definition}

This matrix, $\mathcal{O}_{i,0:n-1}$, is made up of observations from robot $i$ from some initial time $0$ until time-step $n-1$. Now, allow the robot to continue for one more observation and form another matrix $\mathcal{O}_{i,1:n}$. Recall from \cref{def:observability} that the output of robot $i$ for successive time-steps can be represented as successive multiplications by the weight matrix $W$. By utilizing this representation, it is clear that $\mathcal{O}_{i,0:n-1}$ and $\mathcal{O}_{i,1:n}$ are related by $W$ as $\mathcal{O}_{i,0:n-1} W = \mathcal{O}_{i,1:n}$. Therefore, when the two observability matrices have full column rank, the weight matrix can be calculated directly from the relationship using least-squares. Furthermore, robot $i$ does not necessarily need to make $n$ observations to compute $W$, rather it only needs to run for enough time-steps to ensure the observability matrix has full column rank. This is characterized in \cite{obs_estimation} by the parameter $\mathcal{L}_i$. Specifically, the observability matrix estimation only needs to run for $\mathcal{L}_i + 1$ time-steps to guarantee full column rank, where $\mathcal{L}_i$ satisfies $0 \leq \mathcal{L}_i < |\mathcal{P}_i|-|\neighbor_i|$. $\mathcal{P}_i$ is defined as $\mathcal{P}_i = \left\{ j|\text{There exists a path from } j \text{ to } i \text{ in } \mathcal{G} \right\} \cup i$. With this in mind, it is actually sufficient to form $\mathcal{O}_{i,0:\mathcal{L}_i+1}$ and $\mathcal{O}_{i,1:\mathcal{L}_i+2}$ and solve $\mathcal{O}_{i,0:\mathcal{L}_i+1} W = \mathcal{O}_{i,1:\mathcal{L}_i+2}$ for $W$.

\section{Resilient Flocking and Dynamic Tracking}

In this section we apply our findings from previous sections to the problem of multi-robot flocking. A team of robots is tasked with arranging around, and following a moving target that is trying to escape. As long as the robot team can travel at least as fast as the target, the target would not normally be able to escape due to the flocking control law. This motivates a malicious attack where adversaries try to push the legitimate robots away from the target and disrupt the flocking control. Note that our finite time results for \emph{FindSpoofedRobots} can be leveraged for problems with a dynamic nature such as this one where trust values can be learned and applied within a finite time window.

\subsection{Characterization of the Threat on Flocking Control}

\begin{figure}[b!]
    \centering
    \includegraphics[width = 0.44\textwidth]{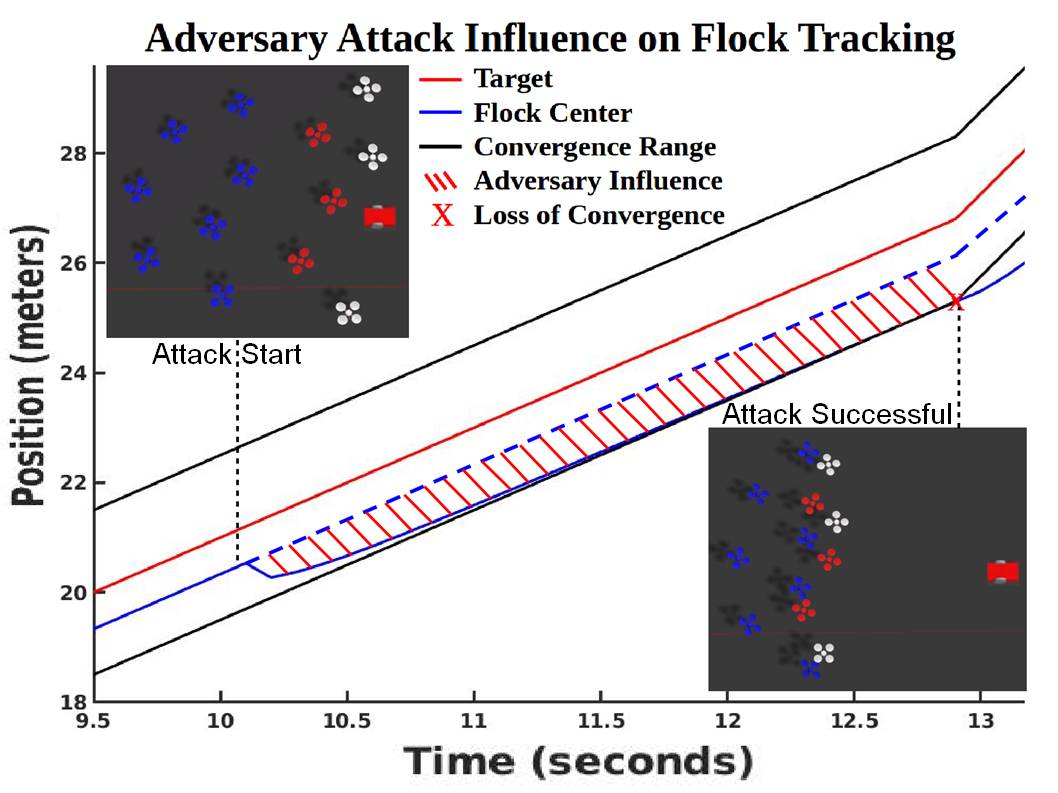}
    \caption{The effect an adversarial attack can have on the actual position of the robots versus the desired position is depicted by the red area. At the red 'x' the flock diverges far enough from the target that the target can escape.}
    \label{fig:adv}
\end{figure}
Malicious robots are able to attack and affect the system if the system trusts the information the robots are transmitting. In order to force the flock to fail and allow the target to escape, the malicious robots must work to strategically drive the legitimate robots out of the Convergence Range as defined in \cref{def:convergence_range}. The way a malicious robot can affect the overall formation of robots is by virtue of the $u_i^{avoid}$ term in Equation \cref{eq:reynolds}. Examples of how the terms in the flocking controller affect each robot are shown in \cref{fig:attack}.

For a maximal attack on the system, malicious robots can try to gain the highest influence on the flocking control scheme; i.e., $u_i \approx u_i^{avoid} >> (u_i^{ref} + u_i^{match})$. Once the malicious robots cause the $u_i^{avoid}$ term to have the highest influence in the controller, they can manipulate the flock however needed. For the case of target tracking, the best attack would be to force the flock in the opposite direction of where the target is moving such as in \cref{fig:attack}.

Next, we analyze how a malicious robot can be positioned strategically to inflict the most damage to the system performance. We do this by analyzing the $u_i^{avoid}$ term in Equation \cref{eq:reynolds}. The term is rewritten in a way that is more insightful about its qualities

\begin{equation}
    u_i^{avoid}[t] = \sum_{j\in\neighbor} \frac{k_{avoid}}{||p_i[t] - p_j[t]||} \frac{\left( p_i[t] - p_j[t] \right)}{||p_i[t] - p_j[t]||}.
    \label{eq:avoid_term}
\end{equation}
From \cref{eq:avoid_term} it is obvious that the velocity input aimed at collision avoidance consists of a magnitude term and a direction term. This implies that in order for any malicious robot to affect a legitimate robot the most, it should occupy a space in the formation that would effectively influence legitimate robots to move to the areas of the formation farther away from the tracked vehicle. A spawning adversary could amplify the effect of this malicious behavior by spoofing robots to greater influence the legitimate robots.

\subsection{Resilient Controller}
\label{sec:conv_results}

Derivations for exponential convergence for flocking control were inspired by the works in \cite{su2007flocking1} and \cite{su2007flocking2}. The flock center converges on the target as $\Bar{u} \rightarrow \text{vel}_\text{cen}$, where
\begin{equation}
    \Bar{u}[t] = \frac{\sum_{i=1}^n u_i[t]}{n} = k_{ref}(p_\text{cen}[t] - \Bar{p}[t]).
\end{equation}
This can be written as a set of difference equations for $\Bar{p}[t] \text{ and } p_\text{cen} [t]$:

\begin{align}
    \frac{\Bar{p}[t+1]-\Bar{p}[t]}{t_s} &= \Bar{u}[t] = k_{ref}(p_\text{cen}[t] - \Bar{p}[t]), \\
    \frac{p_{\text{cen}}[t+1]-p_{\text{cen}}[t]}{t_s} &= \text{vel}_\text{cen},
\end{align}
where $t_s$ is the sampling time between time-steps. Solving this set of difference equations for $\Bar{p}[t]$ gives
\begin{equation}
    \Bar{p}[t] = p_\text{cen}[t] - \frac{\text{vel}_\text{cen}}{k_{ref}} + \left( \Bar{p}[0] - p_\text{cen}[0] + \frac{\text{vel}_\text{cen}}{k_{ref}} \right) \left( 1 - k_{ref}t_s \right)^t.
\end{equation}
Notice that as time-steps progress $\Bar{p}[t]$ converges exponentially to $p_\text{cen}[t] - \frac{\text{vel}_\text{cen}}{k_{ref}}$ assuming $k_{ref}t_s < 1$. This result guarantees exponential convergence of the flock center to a position near the target so long as the robots are within the Convergence Range, \cref{def:convergence_range}. This is demonstrated one-dimensionally in \cref{fig:const_velo}. In the figure, 10 robots start at random positions behind a target moving across the positive y-axis. They converge to within $\pm 0.75$ meters of the target exponentially as determined and settle to a final position of $p_\text{cen}[t] - \frac{\text{vel}_\text{cen}}{k_{ref}}$. However, this result fails to acknowledge the maximum velocity capability of the robots, $\text{vel}_\text{max}$. To incorporate $\text{vel}_\text{max}$, the control is bounded by a maximum control $u_\text{max} = \text{vel}_\text{max}$. This yields a new result:
\begin{equation}
\Bar{u} =  
 \begin{cases} 
     u_\text{max},  & k_{ref}(p_\text{cen}[t] - \Bar{p}[t]) > u_\text{max} \\
      k_{ref}(p_\text{cen}[t] - \Bar{p}[t]), & k_{ref}(p_\text{cen}[t] - \Bar{p}[t]) \leq u_\text{max}.
   \end{cases}
\end{equation}
\begin{figure}[t!]
    \centering
    \includegraphics[width = 0.48\textwidth]{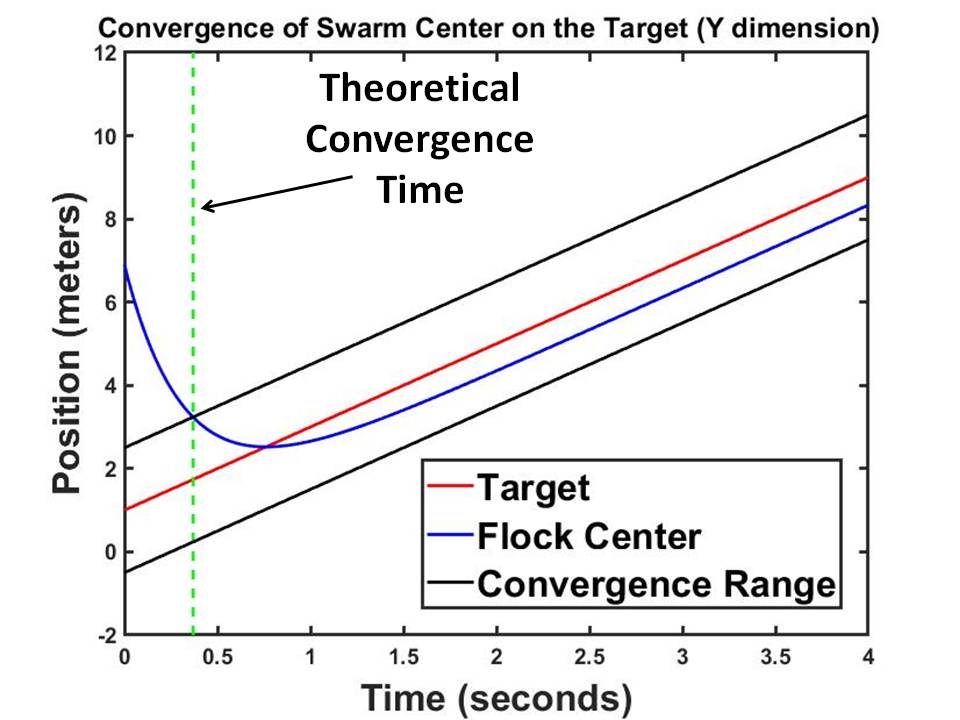}
    \caption{The swarm converges exponentially on the tracked target in the nominal case of no attack}
    \label{fig:const_velo}
\end{figure}
With this more realistic result, we can determine necessary conditions for exponential convergence of the flock center on the target. In order to achieve exponential convergence, the control input must remain within the bounds, namely $|\Bar{u}[t]| \leq u_\text{max}$ for all time-steps $t \geq 0$. This yields a maximum distance the flock center can be from the target at any time

\begin{equation}
    \text{max} |p_\text{cen}[t] - \Bar{p}[t]| = \frac{u_\text{max}}{k_{ref}}.
    \label{eq:exp_conv_range}
\end{equation}
Another condition necessary to ensure convergence is $u_\text{max} > \text{vel}_\text{cen}$. This condition is obvious because if the target is moving faster than the robots can, they will never converge on it.

In the case where adversaries are present, the flock may start to diverge from the target. In order to ensure the flock can still track the target, it is desirable to keep the flock close enough to the target that it can still recover exponential convergence once all adversaries are detected and rejected. From the condition developed in Equation \cref{eq:exp_conv_range}, this means the flock center must not stray more than $\frac{u_\text{max}}{k_{ref}}$ meters from $p_\text{cen}$. It has been shown that while the flock is initially converged on the target, it trails at a distance of $\frac{\text{vel}_\text{cen}}{k_{ref}}$ meters, which means the remaining distance it can stray from the target before the exponential convergence is compromised is $\frac{u_\text{max}}{k_{ref}} - \frac{\text{vel}_\text{cen}}{k_{ref}}$ meters. If the flock is pushed away from the target at its maximum speed, $u_\text{max}$, then the distance between the flock and the target is increasing at a rate of $u_\text{max} + \text{vel}_\text{cen}$ m/s. This information can be used to derive a minimum time required for any adversarial attack to compromise flock tracking of the target:

\begin{equation}
    \Delta = \frac{u_\text{max} - \text{vel}_\text{cen}}{k_{ref} \left( u_\text{max} + \text{vel}_\text{cen} \right)}.
    \label{eq:t_min}
\end{equation}
This phenomenon is represented in \cref{fig:time_to_comp}. As long as \emph{FindSpoofedRobots} can find and reject all spoofed robots in less time than this allowable time window, $\Delta$, the flock can always recover from an attack and avoid losing the target.

An important distinction to note is that the time for \emph{FindSpoofedRobots} to run is denoted in \emph{rounds}, while the time for $\Delta$ is in seconds. It has been stated, in \cref{def:round}, that a round is the amount of time it takes for any message broadcasted to reach all other neighboring robots. This time varies on different robots depending on how the signal is transmitted and the maximum distance robots are from each other. So, the flock can recover from a maximal attack as long as the following inequality holds:
\begin{equation}
    t[r^*]-t[0] \leq \Delta,
\end{equation}
where $t[r^*]-t[0]$ represents the amount of time, in seconds, required for $r^*$ rounds to transmit.

\begin{figure}[h!]
    \centering
    \includegraphics[width = 0.42\textwidth]{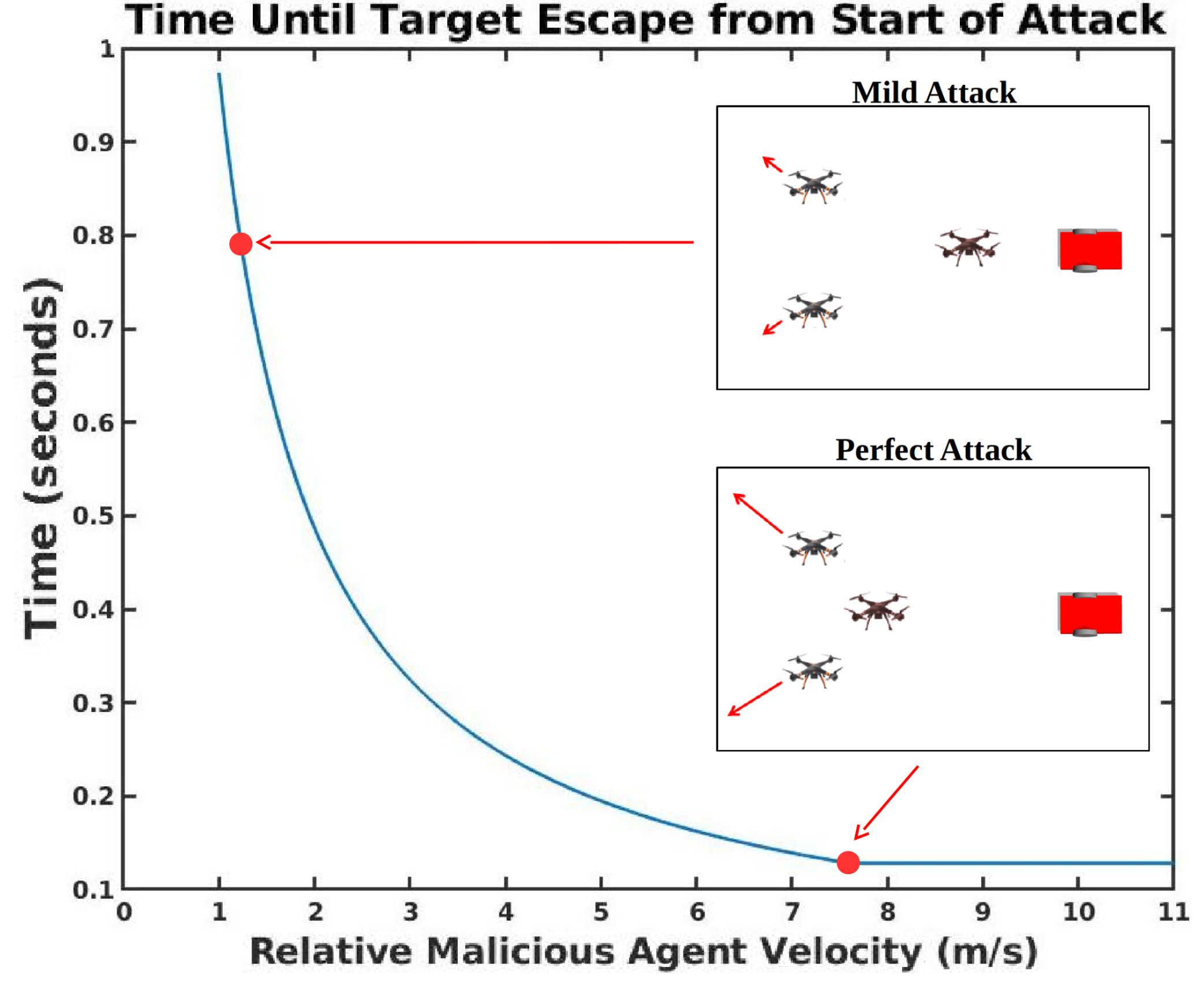}
    \caption{The target cannot escape the flock arbitrarily fast. There is a lower bound on the time required for escape. A strong attack involves the adversary impinging closely on the legitimate robots to drive them away quickly, represented by the longer red arrows in the \emph{Perfect Attack} box.}
    \label{fig:time_to_comp}
\end{figure}

Simulation results are shown below in \cref{sec:simulations} where a comparison to other algorithms such as W-MSR is studied.

\section{Simulation Results}
\label{sec:simulations}

In this section we test the effectiveness of \emph{FindSpoofedRobots} compared to the Baseline spoof resiliency algorithm \cite{AURO} from our previous work. Other qualities of the \emph{FindSpoofedRobots} algorithm with respect to the $\tau$-triangularity of a network are also demonstrated. The presented algorithms are used to show applications to existing theory \cite{WMSR, Sundaram} as well as the specific example of multi-robot flocking with dynamic target tracking. Simulations were developed using the Robot Operating System (ROS) with the Gazebo simulator and results were analyzed in MATLAB.

\subsection{Observing the Added Benefit of Opinion Dynamics}

\cref{fig:comp_to_BM} represents the results of simulations that were run using the Baseline algorithm \cite{AURO} and \emph{FindSpoofedRobots} to demonstrate the improvement in speed of \emph{FindSpoofedRobots} for correctly identifying spoofed robots. The Baseline algorithm \cite{AURO} uses the same observations over the wireless channel, but does not utilize opinions from a robot's neighborhood. This is similar to running only lines 1 and 2 of \emph{FindSpoofedRobots}, i.e., before neighbor opinions are used. Both algorithms run a small number of rounds and produce an output trust vector. The output trust vector of each algorithm is compared with the true trust vector (the vector that would be obtained if there were perfect observations with $\varepsilon = 0.5$ in \cref{eq:e1}, \cref{eq:e2}), and the result is deemed a success if the output trust vector matches the true trust vector. Both algorithms continue to run, each time with more rounds than the last, until they terminate with a successful output trust vector. The plot shows the added benefit of utilizing the opinion of trusted neighbors, as \emph{FindSpoofedRobots} needs to run for less rounds before successful termination at each level of observation quality. This benefit is more pronounced when there are more neighbors for each robot to utilize, as shown in the bottom plot in \cref{fig:comp_to_BM}.

\begin{figure}[t!]
    \centering
    \includegraphics[width = 0.38\textwidth]{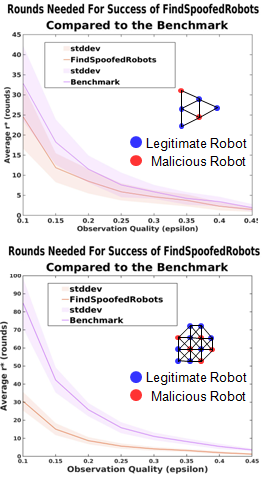}
    \caption{The average amount of rounds needed for \emph{FindSpoofedRobots} and the Baseline algorithm \cite{AURO} to terminate successfully over 1000 simulations with different levels of observation quality. A successful termination occurs when the output vector, $\mathbf{v}_i^*$, for all $i \in L$ robots are equal to the ground truth trust vectors. The sub-networks inside each plot represent the average neighborhood a legitimate robot will see during the simulations. The proportion of legitimate to malicious robots remains similar, but robots in the bottom plot enjoy a larger neighborhood to compare opinions with, demonstrating the utility of \emph{FindSpoofedRobots}.}
    \vspace{-3mm}
    \label{fig:comp_to_BM}
\end{figure}
\vspace{-3mm}

\subsection{A Study on $\tau$-Triangular Networks}

An important aspect of any network running \emph{FindSpoofedRobots} is how $\tau$-triangular the network is. A network with a minimum $\tau$ of 2, i.e., any graph with a chain in it, would require much more rounds of communication to confidently determine the true trust vectors, since the robots in the chain could not utilize information of any neighbors to fortify their opinion on any other neighbors in the chain. However, as the number of $\tau$-triangular neighbors increases, the power of trusted neighbor opinions becomes more evident in fortifying a robot's opinion. This result is reflected in \cref{fig:rstar_vs_tau} where the average number of rounds needed in simulation to converge to the true trust vectors decreases as the minimum $\tau$-triangularity of the graph increases. This idea is further conveyed with \cref{fig:delta_vs_tau_m_const} which shows an improved success rate for \emph{FindSpoofedRobots} as the minimum $\tau$ increases when the number of rounds remains constant.

\begin{figure}[h!]
    \centering
    \includegraphics[width = 0.45\textwidth]{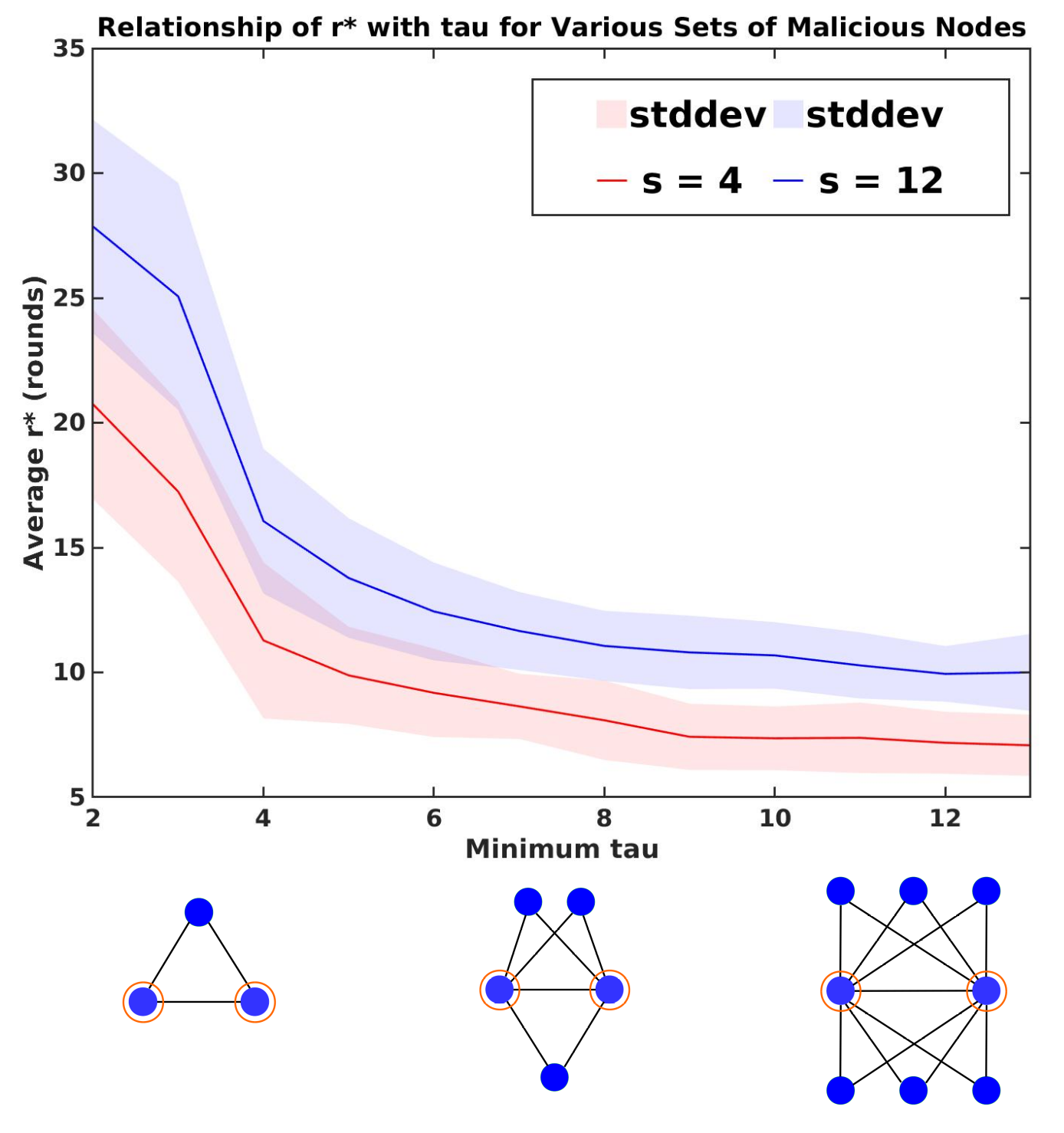}
    \caption{The relationship of r* with $\tau$ as the minimum $\tau$ in the network increases. 
    1000 simulations were run at each level of minimum $\tau$-triangularity for a network with $s = 4$ spawning and spoofed robots, and $s = 12$ spawning and spoofed robots, with $\varepsilon = 0.1$.}
    \label{fig:rstar_vs_tau}
\end{figure}
\vspace{-4mm}
\begin{figure}[h!]
    \centering
    \includegraphics[width = 0.45\textwidth]{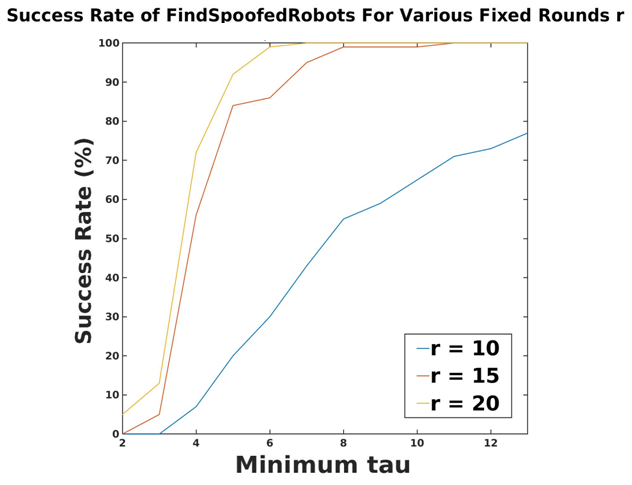}
    \caption{As the minimum $\tau$ increases the success rate of \emph{FindSpoofedRobots} increases for a fixed number of rounds. 1000 simulations were run for a fixed number of rounds $r = 10$, $r = 15$, and $r = 20$ and each minimum $\tau$ value, with $s = 8$ malicious robots and $\varepsilon = 0.2$.}
    \label{fig:delta_vs_tau_m_const}
\end{figure}

Since an increased minimum $\tau$-triangularity is desired for faster results, it is important to understand how the $\tau$-triangularity may vary in different networks. It is reasonable to assume that in a flocking scenario where the robots tend to condense into a formation, that communication triangles are likely to occur. However, \cref{fig:minTau_vs_N} shows that these communication triangles occur in random networks as well. In \cref{fig:minTau_vs_N}, robots are placed into boxes of various sizes at random locations, forming a random communication network. Inter-robot communication for this study was stochastic with communication happening at a probability of 1 if the robots were inside a desired communication radius, $||p_i - p_j|| \leq 1$, a probability of $\frac{1}{||p_i - p_j||}$ when the distance is more than 1 meter away but less than 4 meters away, and a probability of 0 otherwise. Simulations suggest that with a low number of robots placed inside the box, there are bound to be less communication triangles, but as the number of robots in the box increases, the number of triangles also increases. Furthermore, the approximately constant slopes of the curves in \cref{fig:minTau_vs_N} suggest that there is a near direct relationship between the number of robots in the box and the minimum $\tau$ of the network, if we assume that communication opportunities increase with smaller inter-robot distances, which is reasonable to assume. This means larger networks would naturally have a higher minimum $\tau$ value.

\begin{figure}[h!]
    \centering
    \includegraphics[width = 0.45\textwidth]{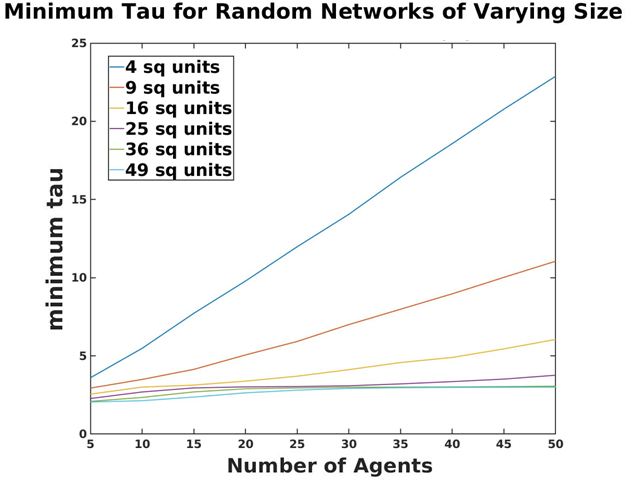}
    \caption{The robots are confined to boxes of different sizes. The relationship between the number of robots randomly placed in the box and the average minimum $\tau$-triangularity of the resultant network is observed. 1000 simulations were run for each number of robots at each box size. For this example, edges were determined using the following rule: communication occurs between robots $i$ and $j$ with probability 1 if $||p_i - p_j|| \leq 1$, with probability $1/||p_i - p_j||$ when $1 < ||p_i - p_j|| < 4$, and with probability 0 otherwise.}
    \label{fig:minTau_vs_N}
\end{figure}
\vspace{-4mm}
\subsection{Reinforcing the W-MSR Algorithm With \emph{FindSpoofedRobots}}

Simulations were run with the example robot network shown earlier in \cref{fig:with_spoofing} to demonstrate how \emph{FindSpoofedRobots} can add resilience against spoofed adversaries. Recall in \cref{sec:flocking_consensus} that robots often may want to share information with neighbors in order to reach an agreement. This agreement could be on a position, desired goal, etc. For this example, consider a problem where a team of robots must agree upon the position of a static target in order to arrange around it. All robots can determine a noisy estimation of the target position, but should cooperate to strengthen their estimations and agree on a common value. The robots can utilize theory from \cite{WMSR} to try to add some resiliency to adversaries.

In these simulations, the robots will use the W-MSR algorithm \cite{WMSR} to gain some resiliency to adversaries. The network in \cref{fig:with_spoofing} with two spoofed adversaries is perceived to be ($2,2$)-robust, so W-MSR can handle up to 1 adversary. The two spoofed adversaries is too much for W-MSR alone to handle. This is demonstrated in \cref{fig:W-MSR_1adv}.
\begin{figure}[h!]
    \centering
    \includegraphics[width = 0.45\textwidth]{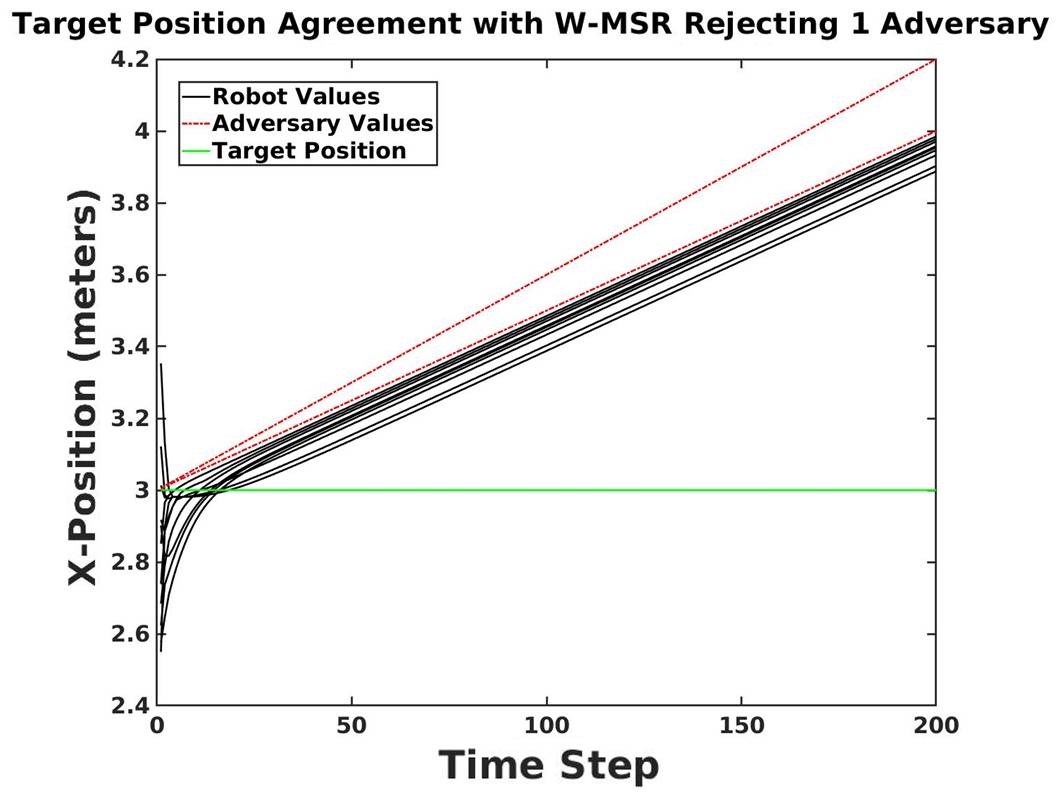}
    \caption{The flock tends to the value of the second spoofed adversary}
    \label{fig:W-MSR_1adv}
\end{figure}
In \cref{fig:W-MSR_1adv}, the robots successfully reject the influence of one adversary, but they still tend toward the value of the second adversary, which drives the flock to consensus on a target position that is increasingly farther from the true position. However, if \emph{FindSpoofedRobots} is used first to determine the trustworthy network in \cref{fig:without_spoofing}, then the robots can reject the influence of all spoofed adversaries, and achieve consensus on the target as shown in \cref{fig:no_adv}.
\begin{figure}[h!]
    \centering
    \includegraphics[width = 0.45\textwidth]{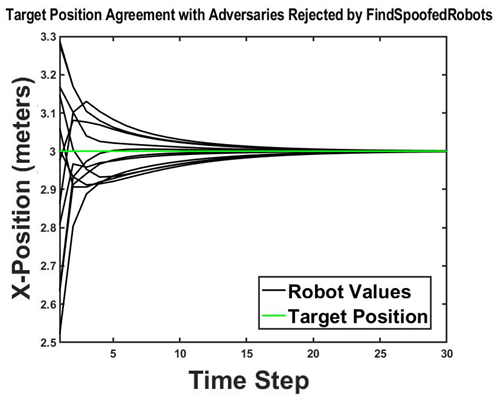}
    \caption{With adversaries ignored, the flock can accurately track the target}
    \label{fig:no_adv}
\end{figure}
\vspace{-5mm}

\subsection{Flocking with Dynamic Target Tracking Results}

\cref{fig:no_escape} tests the flock tracking with adversaries. This simulation was run with $n = 13$ robots. The gains used were $k_{ref} = 3$, $k_{avoid} = 1$, and $k_{match} = 0.2$. The ratio of $k_{ref}:k_{avoid}$ was chosen experimentally to maintain a desired spacing between converged robots. The target moves along the x-axis at 2 m/s from an initial position of ($1$, $1$, $0$) meters. The robots start at random initial positions behind the target. The target and the robots are limited to a maximum velocity of 4.5 m/s. At $t = 10$ seconds the three right-most robots are hacked. Each hacked robot then spawns one spoofed robot near the closest legitimate robot to it. The 6 malicious robots then start to slowly push the legitimate robots out of the exponential convergence range (shown as black lines bordering the target position). If \emph{FindSpoofedRobots} were not used, the malicious robots could successfully allow the target to escape, as the legitimate robots would not want to crash into the other robots. This result was shown earlier in \cref{fig:adv}, with the time of convergence loss denoted with a red `x'. \cref{fig:no_escape} depicts what happens when \emph{FindSpoofedRobots} is used and is able to find and reject the spoofed robots in time before the target escapes the Convergence Range. Once the spoofed robots are rejected, the legitimate robots can successfully evade the three hacked robots and continue to track the target, since the three hacked robots cannot occupy enough space to successfully block the other robots from advancing.

\begin{figure}[h!]
    \centering
    \includegraphics[width = 0.42\textwidth]{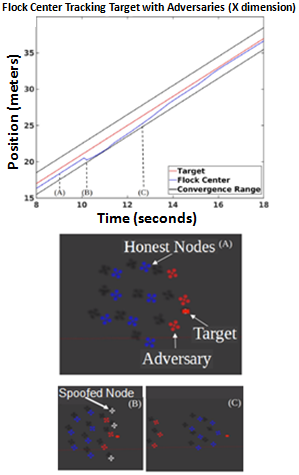}
    \caption{The flock successfully rejects spoofed robots quickly and never loses the target}
    \label{fig:no_escape}
\end{figure}
\vspace{-4mm}
\section{Conclusion} 
This paper presents novel trust algorithms that utilize neighbor's opinions to strengthen convergence to the true set of legitimate and adversarial robots. We develop a theoretical convergence time in \emph{rounds} for the method with respect to the size of the multi-robot network, the size of the threat, and the quality of information gathered. The utility of the algorithms is presented using a multi-robot flocking and dynamic target tracking example. The algorithms are also compared to several pre-existing solutions in the literature, and their relevance is discussed. Furthermore, a convergence requirement for the flock to track a target was derived. The results of the algorithms were extensively tested and simulation shows its effectiveness and importance for multi-robot network resilience to adversaries.

%

\appendix
\subsection{Proofs}\label{sec:proofs}

\begin{proof}[Proof of \cref{lem:maj}]
Let $Y_{\rho}=\sum_{t=1}^\rho X_{j}^t(i)$. By assumption the $\{X_{j}^t(i)\}_{t\geq 1}$ are independent.
Without loss of generality assume $j$ is illegitimate (the other case is analogous).
Thus $\E{ Y_\rho } = r(1/2-\varepsilon)$. 
Our goal is to use Azuma-Hoeffding (\cref{thm:Azuma}), which requires the underlying variable be a Martingale. Note that $Y_\rho$ is not a Martingale, however
$Z_\rho = Y_\rho - \sum_t^\rho \E{X_{j}^t(i)} $ is. Thus, we will apply Azuma-Hoeffding to $Z_\rho$. Note that $|Z_t-Z_{t-1}|\leq 1$.
\vspace{-2mm}
\begin{align*}
    &\Pr{\text{$i$ believes $j$ is legitimate}}= \Pr{Y_\rho  \geq \frac{r}{2}} \\
     &= \Pr{Z_\rho \geq \E{Y_\rho}+\varepsilon r} \leq \exp{\left(- \frac{2(\varepsilon r)^2}{ r }\right) }\\
     &=\exp(-2\varepsilon^2 r)= \delta.
\end{align*}
The other case, where $j$ is legitimate, is analogous.
\end{proof}

\begin{proof}[Proof of \cref{lem:lowermaj}]
Consider a legitimate robot $i$.
We set the $X_j^t(i)$  such that
for $j\in L$, $X_i^t(j)=1$ w.p. $1/2+\varepsilon$ and $0$ otherwise.
Let $\mathcal{D}^+$ denote this probability distribution.
Furthermore, for $j\in A_{spoof} \cup A_{spawn}$, $X_j(i)^t=1$ w.p. $1/2-\varepsilon$ and $0$ otherwise.
Let $\mathcal{D}^-$ denote this probability distribution.
Note that this satisfies the assumptions \cref{eq:e1} and \cref{eq:e2}.
Therefore, in order for $i$ to compute w.p. at least $1-\delta$, the correct trust vector entry for $j$ having probability distribution $\mathcal{D}^{unknown}$, $i$ needs to distinguish between the case that
$\mathcal{D}^{unknown} =  \mathcal{D}^-$ and the case
$\mathcal{D}^{unknown} = \mathcal{D}^+$  w.p. at least $1-\delta$.

Note that the total variation distance between $\mathcal{D}^+$ and 
$\mathcal{D}^-$ is
\[ TV(\mathcal{D}^+,\mathcal{D}^-)=\frac12 \cdot 2|(1/2+\varepsilon - (1/2-\varepsilon) ) )= 2\varepsilon\]

Therefore, by \cref{thm:ICALPlower}, this requires $c \log (1/\delta)/\epsilon^2$ rounds, for a suitable constant $c$. Note that in our specific setting, the domain is indeed discrete. 
Let $\delta=1/n$.
Hence, the probability that all robots have the correct final trust vector is, by using the independence,   $(1-\delta)^n = (1-1/n)^n \leq 1/e $.
Hence more than $\log(n)/\epsilon^2$ rounds are necessary.
\end{proof}

\begin{proof}[Proof of \cref{thm:one}]
Consider a legitimate robot $i$.

Let $\delta'=\frac{\delta \varepsilon^2 \tau}{ 2 l n}$ and $\mindelta = \frac{  {\delta'}^{\frac{\magic}{\tau}}}{2}  \left(\frac{\tau}{2e^e d_{L}}\right)^2$.
By our assumption $r^*\geq   \log(1/\mindelta)/(2\varepsilon^2)$ and hence, by \cref{lem:maj}, each interim trust vector entry is correct w.p. at least $1-\mindelta$. Note that, it is unlikely that all entries are correct, though.
%
%
%
%
%
%
Recall that $v_{i,k}$ is the indicator variable that is $1$ if $i$ trusts robot $k$.
In order for $v_{i,j}^*$ to be set correctly, a sufficient condition is that
\begin{align}\label{eq:thathere}  
\sum_{k\in \neighbor_i \cap \neighbor_j\cap L} \hspace{-0.3cm} v_{i,k}    v_{k,j} 
> \hspace{-0.7cm}
\sum_{ k \in \neighbor_i \cap \neighbor_j\cap A_{hid}} \hspace{-0.5cm}1 
+ \sum_{ \hspace{0.3cm} k \in \neighbor_i \cap \neighbor_j\cap \left(A_{spawn} \cup A_{spoof}\right)}  \hspace{-1cm} v_{i,k}  v_{k,j} .
\end{align}
In the following we will show \cref{eq:thathere}
holds with probability at least $p:= 1-2\delta'$. 
Recall, $l_{i,j} = |\neighbor_i \cap
\neighbor_j \cap L|$, 
$s_{i,j} = |\neighbor_i \cap \neighbor_j\cap \left(A_{spawn} \cup A_{spoof}\right)|$
and $h_{i,j}=|\neighbor_i \cap
\neighbor_j \cap A_{hid}|$.

By the above, the probability that either of the values $v_{i,k}$ or $v_{k,j}$ is set incorrectly is at most $2\mindelta$, by union bound.
Let $\mathcal{X}\sim \operatorname{BIN}(l_{i,j}, 1-2\mindelta )$ be drawn from the binomial distribution. Note that
$\sum_{k\in \neighbor_i \cap \neighbor_j\cap L} v_{i,k}  v_{k,j} $ majorizes $\mathcal{X}$. 
Let $\mathcal{Y}\sim \operatorname{BIN}(m_{i,j}, \mindelta)$. 
Note that
$\mathcal{Y}$ majorizes $ \sum_{k \in \neighbor_i \cap \neighbor_j\cap \left(A_{spawn} \cup A_{spoof}\right)} v_{i,k}  v_{k,j}$  and hence it is sufficient to prove that with probability $p$
\[ \mathcal{X} >   h_{i,j} + \mathcal{Y}, \]
since this implies \cref{eq:thathere}.
Let $\mathcal{X}'\sim \operatorname{BIN}(l_{i,j},2\mindelta)$.
We have $\E{\mathcal{X} } = l_{i,j} (1-2\mindelta)$, $\E{ \mathcal{X}' } =2\mindelta l_{i,j}$ and
$\E{\mathcal{Y} }  = \mindelta m_{i,j}$. 
Let $\alpha = \tau/(2l_{i,j}) \geq  \tau/(2d_{L}) $.
Observe that 
$2\mindelta\leq  1/\left(\alpha/e^e \right)^2 $.

Therefore,
$1/(2\mindelta) \geq \left(e^e/\alpha \right)^2 $.
Hence,
$1/(\sqrt{2\mindelta)} \geq e^e/\alpha $
and thus
$\log(1/(\sqrt{2\mindelta)}) \geq (\log(1/\alpha) + e) $ and therefore
$\log(1/(2\mindelta)) \geq 2(\log(1/\alpha) + e) $.
By, \cref{thm:bounds-binomial-distribution},  we get
\begin{equation} \begin{split} & \Pr {\mathcal{X} \geq l_{i,j} -\tau/2) } 
= \Pr{ \mathcal{X}' \leq \tau/2 }= \Pr{ \mathcal{X}' \leq \alpha l_{i,j} }\\ 
& < \exp\left(-  \alpha l_{i,j} \log\left(\frac{1}{2\mindelta}\right) 
    + \alpha l_{i,j} \log\left(\frac{1}{\alpha}\right) 
    + e\alpha l_{i,j} \right) 
\\ &
\leq \exp\left(-  \frac{\tau \log(1/(2\mindelta))}{\magic}  \right)
\leq \exp\left(-  \frac{\tau \log(1/{\delta'}^{(\magic/\tau)})}{\magic}  \right)\\
&= \delta'.
\end{split} \end{equation}
Let $\alpha' = \frac{\tau}{2 m_{i,j}}$.
Similarly, as before, $\log(1/(2\mindelta)) \geq 2(\log(1/\alpha') + e) $.
By, \cref{thm:bounds-binomial-distribution},
\begin{equation} \begin{split} 
&\Pr {  \mathcal{Y} \leq \tau/2} =  \Pr {  \mathcal{Y} \leq \alpha'm_{i,j}}  \\
& < \exp\left(-  \alpha' m_{i,j} \log\left(\frac{1}{\mindelta}\right) + \alpha' m_{i,j} \log\left(\frac{1}{\alpha'}\right)  + e\alpha'  m_{i,j} \right) 
\\ &
\leq \exp\left(-  \frac{\tau \log(1/\mindelta)}{\magic}  \right)
\leq \exp\left(-  \frac{\tau \log(1/{\delta'_i}^{(\magic/\tau)})}{\magic}  \right)\\
&= \delta'.
\end{split} \end{equation}
By assumption, 
$\tau >0 $
Thus, by Union bound, with probability at least $ p$,
\[
 \mathcal{X}> l_{i,j} -\tau/2  \geq h_{i,j} + \tau/2 >  h_{i,j} +  \mathcal{Y},
\]
where we used that $l_{i,j} - h_{i,j} \geq \tau$, by definition of $\tau$.

We have shown that for any pair of legitimate robots $i$ and any neighboring robot $j$, $i$th final trust vector entry for $j$ after round $r^*$ is correct w.p. $1-2\mindelta$. 
Taking union bound over all of all $i\in L$ and all robots $j$ yields that \emph{every} legitimate robot has a correct final trust vector w.p. at least 
$1-2ln\mindelta$.

It remains to argue that the results remains true for all $r > r^*$. 
%
In the above, we have argued that the failure probability is 
$\delta^* := 2ln\delta'$ after $r^*=  \log( 1/\delta') f + g$ rounds, where $f$ and $g$ are appropriate functions that do not depend on $\delta'$. 
From this, we can derive that in order to reduce the error to $\delta^*/e^t$, we need
$\log(e^t/\delta') f + g=(\log(1/\delta')+\log(e^t))f +g=
r^* + t\cdot f
$ rounds. 
Thus, at round $r=r^*+ kf$ we
get an error probability of $\delta^*/e^k$.
This allows us to consider intervals of length $2/\epsilon^2$
%
Taking union bound over all rounds $r \geq r^*$ we get an error probability of 
\[ \sum_{k}  \frac{\delta^*}{e^k} f= 
f \delta^*
\sum_{k} \frac{1}{e^k} \leq 2f ln\delta' = \delta
. \]

\end{proof}

\vspace{-2mm}
\subsection{Auxiliary Claims}

The following is a slight modification  of Thm. 5.2 in \cite{dubhashi2009concentration}.
\begin{theorem}[Azuma-Hoeffding inequality - general version~\cite{dubhashi2009concentration}] \label{pro:hoeff}\label{thm:Azuma}	
	Let $Y_0, Y_1,\dots$ be a martingale with respect ot the sequence $X_0,X_1,\dots$. Suppose also that $Y_i$ satisfies $a_i \leq Y_i -Y_{i-1}\leq b_i$ for all $i$.
	Then, for any $t$ and $n$
	\[
		\Pr{Y_n - Y_0 \geq t },
		\Pr{Y_0 - Y_n \geq t } \leq  \exp\left(\frac{-2t^2}{\sum_{i=1}^n (b_i -
		a_i)^2}\right).
		\]
\end{theorem}

\begin{definition}[\cite{ICALPlower}]\label{def:ICALPlower}
Given a target distribution $q$ with
domain $D$ of size $N$, parameters $0 < \epsilon, \delta < 1$, and sample access to an unknown distribution
$p$ over the same domain, we want to distinguish with probability at least $1-\delta$ between the
following cases:
 $p = q$ and $TV (p, q)\geq \epsilon$, where $TV(\cdot,\cdot)$ denotes the total variation distance.
We call this the problem of $(\epsilon,\delta$)-testing identity to $q$. 
\end{definition}
\begin{theorem}[\cite{ICALPlower}]\label{thm:ICALPlower}
There exists a computationally efficient $(\epsilon,\delta)$-identity tester
for discrete distributions of support size $n$ with sample complexity
$\Theta\left(\frac{1}{\epsilon}\left( 
\sqrt{n\log(1/\delta)} + \log(1/\delta)
\right)\right)$
Moreover, this sample size is information-theoretically optimal, up to a constant factor, for
all $n,\epsilon,\delta$.
\end{theorem}

\begin{theorem}\cite[Equation 10]{HR90}\label{thm:bounds-binomial-distribution}
Let $Y=\sum_{i=1}^m Y_i$ be the sum of $s$ independent and identically distributed random variables with
$\Pr { Y_i=1} =p$ and $\Pr { Y_i=0 } =1-p$. We have for any
$\alpha \in (0,0.8]$ that
\begin{multline}
\Pr {Y \geq \alpha \cdot m } \\ \leq \exp\left(  -\alpha m  \log (1/p) + \alpha m  \log (1/\alpha)   + e\alpha(1-\alpha) m  \right).
\end{multline}
\end{theorem}
\begin{proof}
We have, by {\cite[Equation 10]{HR90}} and simple calculations,
\begin{equation} 
\begin{split}%
& \Pr{ Y \geq \alpha \cdot m } \\ &\leq \left(\left(\frac{p}{\alpha} \right)^\alpha \left(\frac{1-p}{1-\alpha} \right)^{1-\alpha}\right)^m \leq  
 \left(\left(\frac{p}{\alpha} \right)^\alpha \left(\frac{1}{1-\alpha} \right)^{1-\alpha}\right)^m \\
 & \leq
 \exp\left(  
 \alpha m  \log \left(\frac{p}{\alpha}\right)  + (1-\alpha) m \log \left(\frac{1}{1-\alpha}\right) \right)  \\
 &=  \exp \left(-\alpha m  \log \left(\frac{1}{p}\right) + \alpha m  \log \left(\frac{1}{\alpha}\right) 
 + (1-\alpha) m \log\left(\frac{1}{1-\alpha} \right) \right) \\
 & < \exp\left(  -\alpha m  \log \left(\frac{1}{p}\right) + \alpha m  \log \left(\frac{1}{\alpha}\right)   + e\alpha(1-\alpha) m  \right),
 \end{split} \end{equation}
 where we used that $\log(1/(1-x))\leq ex $ for $x\in (0,0.8]$.
\end{proof}

\begin{table*}[h!]
\centering
\caption{Notation}
\begin{tabular}{|c|c|c|c|} \hline
     $X$ & Observations over Wireless Channels & $i,j,k$ & Indices \\ \hline
     $t$ & Time & $P$ & Set of Positions \\ \hline
     $p_i[t]$ & Position of Robot $i$ at time t & $\text{VEL}$ & Set of Velocities \\ \hline
     $\text{vel}_i[t]$ & Velocity of Robot $i$ at time t & $n$ & Number of Robots \\ \hline 
     $\text{vel}_\text{max}$ & Maximum Robot Velocity & $\neighbor_i$ & Neighborhood of Robot $i$ \\ \hline 
     $L$ & Set of Legitimate Robots & $M$ & Set of Malicious Robots  \\ \hline
     $l$ & Cardinality of $L$ & $m$ & Cardinality of $M$  \\ \hline
     $A_{spawn}$ & Set of Spawning Adversaries & $A_{spoof}$ & Set of Spoofed Robots \\ \hline 
     $S$ & Set of Detectable Adversaries ($A_{spawn} \cup A_{spoof}$) & $A_{hid}$ & Set of Hidden Adversaries \\ \hline 
     $s$ & Cardinality of $S$ & $h$ & Cardinality of $A_{hid}$ \\ \hline
     $\tau$ & Gap Between Legitimate and Hidden Adversaries  & $\varepsilon$ & Quality of Observation \cref{eq:e1} \\ \hline 
     $\mathbf{v}$, $\mathbf{v}^*$ & Trust Vector (interim, final) & $v$, $v^*$ & Entries in (interim, final) Trust Vector \\ \hline
     $r$ & Round & $r^*$ & Number of Rounds to Reach Consensus \\ \hline
     $\delta$ & Failure Probability & $x$ & Consensus Value \\ \hline
     $W$ & Weight Matrix & $B$ & Input Matrix \\ \hline
     $u$ & Input & $e$ & Canonical Vector and Euler Constant \\ \hline $C$ & Output Matrix & $y$ & Output Consensus Value \\ \hline $\mathcal{O}$ & Observability Matrix & $\text{cen}$ & Target (center) \\ \hline 
     $k_{ref},k_{avoid},k_{match}$,$c$ & Constants & $\Delta$ & Minimum Time for Convergence Loss \\ \hline 
     $\mathbf{V}^*$ & Matrix of Trust Vectors & $\mathcal{G}$ & Graph Adjacency Matrix \\ \hline
     $\mathcal{L}$ & Observability Column Rank Parameter & $\mathcal{P}$ & Paths from $j$ to $i$ \\ \hline
\end{tabular}
\label{tab:notation}
\end{table*}

\section*{Acknowledgment}
The authors gratefully acknowledge partial funding support from the MIT Lincoln Laboratories Line Project.


\ifCLASSOPTIONcaptionsoff
  \newpage
\fi



%
\bibliographystyle{IEEEtran}
\bibliography{references}

\begin{thebibliography}{10}
\providecommand{\url}[1]{#1}
\csname url@samestyle\endcsname
\providecommand{\newblock}{\relax}
\providecommand{\bibinfo}[2]{#2}
\providecommand{\BIBentrySTDinterwordspacing}{\spaceskip=0pt\relax}
\providecommand{\BIBentryALTinterwordstretchfactor}{4}
\providecommand{\BIBentryALTinterwordspacing}{\spaceskip=\fontdimen2\font plus
\BIBentryALTinterwordstretchfactor\fontdimen3\font minus
  \fontdimen4\font\relax}
\providecommand{\BIBforeignlanguage}[2]{{%
\expandafter\ifx\csname l@#1\endcsname\relax
\typeout{** WARNING: IEEEtran.bst: No hyphenation pattern has been}%
\typeout{** loaded for the language `#1'. Using the pattern for}%
\typeout{** the default language instead.}%
\else
\language=\csname l@#1\endcsname
\fi
#2}}
\providecommand{\BIBdecl}{\relax}
\BIBdecl

\bibitem{zhou2018resilient}
L.~Zhou, V.~Tzoumas, G.~J. Pappas, and P.~Tokekar,
  ``\href{https://ieeexplore.ieee.org/document/8534468}{Resilient active target
  tracking with multiple robots},'' \emph{IEEE Robotics and Automation
  Letters}, vol.~4, no.~1, pp. 129--136, 2018.

\bibitem{lewis2019developing}
M.~Lewis, K.~Sycara, and I.~Nourbakhsh,
  ``\href{http://citeseerx.ist.psu.edu/viewdoc/download?doi=10.1.1.11.859&rep=rep1&type=pdf}{Developing
  a testbed for studying human-robot interaction in urban search and rescue},''
  \emph{Proceedings of the 10th International Conference on Human Computer
  Interaction (HCII'03)}, pp. 270--274, 2019.

\bibitem{1678135}
B.~{Grocholsky}, J.~{Keller}, V.~{Kumar}, and G.~{Pappas},
  ``\href{https://ieeexplore.ieee.org/stamp/stamp.jsp?tp=&arnumber=1678135}{Cooperative
  air and ground surveillance},'' \emph{IEEE Robotics Automation Magazine},
  vol.~13, no.~3, pp. 16--25, Sep. 2006.

\bibitem{FAA}
U.~S.~D. of~Transportation Federal Aviation~Administration,
  \emph{\href{https://www.govinfo.gov/content/pkg/FR-2019-12-31/pdf/2019-28100.pdf}{Remote
  Identification of Unmanned Aircraft Systems}}, 2019.

\bibitem{8424544}
A.~A. {Paranjape}, S.~{Chung}, K.~{Kim}, and D.~H. {Shim},
  ``\href{https://ieeexplore.ieee.org/stamp/stamp.jsp?tp=&arnumber=8424544}{Robotic
  Herding of a Flock of Birds Using an Unmanned Aerial Vehicle},'' \emph{IEEE
  Transactions on Robotics}, vol.~34, no.~4, pp. 901--915, Aug 2018.

\bibitem{6815901}
M.~{Strohmeier}, M.~{Schäfer}, V.~{Lenders}, and I.~{Martinovic},
  ``\href{https://ieeexplore.ieee.org/stamp/stamp.jsp?tp=&arnumber=6815901}{Realities
  and challenges of nextgen air traffic management: the case of ADS-B},''
  \emph{IEEE Communications Magazine}, vol.~52, no.~5, pp. 111--118, May 2014.

\bibitem{664154}
C.~{Tomlin}, G.~J. {Pappas}, and S.~{Sastry},
  ``\href{https://ieeexplore.ieee.org/stamp/stamp.jsp?tp=&arnumber=664154}{Conflict
  resolution for air traffic management: a study in multiagent hybrid
  systems},'' \emph{IEEE Transactions on Automatic Control}, vol.~43, no.~4,
  pp. 509--521, April 1998.

\bibitem{10.5555/3306127.3331683}
R.~Liu, F.~Jia, W.~Luo, M.~Chandarana, C.~Nam, M.~Lewis, and K.~Sycara,
  ``\href{https://dl.acm.org/doi/pdf/10.5555/3306127.3331683?download=true}{Trust-Aware
  Behavior Reflection for Robot Swarm Self-Healing},'' \emph{Proceedings of the
  18th International Conference on Autonomous Agents and MultiAgent Systems},
  p. 122–130, 2019.

\bibitem{WMSR}
H.~J. LeBlanc, H.~Zhang, X.~Koutsoukos, and S.~Sundaram,
  ``\href{https://ieeexplore.ieee.org/document/6481629}{Resilient asymptotic
  consensus in robust networks},'' \emph{IEEE Journal on Selected Areas in
  Communications}, vol.~31, no.~4, pp. 766--781, 2013.

\bibitem{Bullo}
F.~Pasqualetti, A.~Bicchi, and F.~Bullo,
  ``\href{https://ieeexplore.ieee.org/document/5779706}{Consensus computation
  in unreliable networks: A system theoretic approach},'' \emph{IEEE
  Transactions on Automatic Control}, vol.~57, no.~1, pp. 90--104, 2011.

\bibitem{Sundaram}
S.~{Sundaram} and C.~N. {Hadjicostis},
  ``\href{https://ieeexplore.ieee.org/document/5605238}{Distributed Function
  Calculation via Linear Iterative Strategies in the Presence of Malicious
  Agents},'' \emph{IEEE Transactions on Automatic Control}, vol.~56, no.~7, pp.
  1495--1508, 2011.

\bibitem{7822915}
K.~{Saulnier}, D.~{Saldaña}, A.~{Prorok}, G.~J. {Pappas}, and V.~{Kumar},
  ``\href{https://ieeexplore.ieee.org/stamp/stamp.jsp?arnumber=7822915}{Resilient
  Flocking for Mobile Robot Teams},'' \emph{IEEE Robotics and Automation
  Letters}, vol.~2, no.~2, pp. 1039--1046, April 2017.

\bibitem{SGcontroller}
M.~C. De~Gennaro and A.~Jadbabaie,
  ``\href{https://ieeexplore.ieee.org/document/4177054}{Decentralized control
  of connectivity for multi-agent systems},'' \emph{Proceedings of the 45th
  IEEE Conference on Decision and Control}, pp. 3628--3633, 2006.

\bibitem{180296}
J.~Xiong and K.~Jamieson,
  ``\href{https://www.usenix.org/system/files/conference/nsdi13/nsdi13-final51.pdf}{ArrayTrack:
  A Fine-Grained Indoor Location System},'' \emph{Presented as part of the 10th
  {USENIX} Symposium on Networked Systems Design and Implementation ({NSDI}
  13)}, pp. 71--84, 2013.

\bibitem{Candell_PhysicsBasedDetection}
J.~Giraldo, D.~Urbina, A.~Cardenas, J.~Valente, M.~Faisal, J.~Ruths, N.~O.
  Tippenhauer, H.~Sandberg, and R.~Candell,
  ``\href{https://www.ncbi.nlm.nih.gov/pmc/articles/PMC6512826/}{A Survey of
  Physics-Based Attack Detection in Cyber-Physical Systems},'' \emph{ACM
  Comput. Surv.}, vol.~51, no.~4, Jul. 2018.

\bibitem{dataVeracityCPS_entropyMethods}
M.~Krotofil, J.~Larsen, and D.~Gollmann,
  ``\href{https://dl.acm.org/doi/abs/10.1145/2714576.2714599}{The Process
  Matters: Ensuring Data Veracity in Cyber-Physical Systems},''
  \emph{Proceedings of the 10th ACM Symposium on Information, Computer and
  Communications Security}, p. 133–144, 2015.

\bibitem{Sinopoli_Watermark}
Y.~{Mo}, S.~{Weerakkody}, and B.~{Sinopoli},
  ``\href{https://ieeexplore.ieee.org/abstract/document/7011170}{Physical
  Authentication of Control Systems: Designing Watermarked Control Inputs to
  Detect Counterfeit Sensor Outputs},'' \emph{IEEE Control Systems Magazine},
  vol.~35, no.~1, pp. 93--109, 2015.

\bibitem{pappasSecretChannelCodes}
A.~{Tsiamis}, K.~{Gatsis}, and G.~J. {Pappas},
  ``\href{https://ieeexplore.ieee.org/abstract/document/8758381}{State-Secrecy
  Codes for Networked Linear Systems},'' \emph{IEEE Transactions on Automatic
  Control}, vol.~65, no.~5, pp. 2001--2015, 2020.

\bibitem{AURO}
S.~Gil, S.~Kumar, M.~Mazumder, D.~Katabi, and D.~Rus,
  ``\href{https://link.springer.com/content/pdf/10.1007\%2Fs10514-017-9621-5.pdf}{Guaranteeing
  spoof-resilient multi-robot networks},'' \emph{Autonomous Robots}, vol.~41,
  no.~6, 2017.

\bibitem{modelingDependability_CPS}
J.~{Lin}, S.~{Sedigh}, and A.~{Miller},
  ``\href{https://ieeexplore.ieee.org/abstract/document/5254189}{A General
  Framework for Quantitative Modeling of Dependability in Cyber-Physical
  Systems: A Proposal for Doctoral Research},'' \emph{2009 33rd Annual IEEE
  International Computer Software and Applications Conference}, vol.~1, pp.
  668--671, 2009.

\bibitem{sastryCPS_survivability}
A.~A. {Cardenas}, S.~{Amin}, and S.~{Sastry},
  ``\href{https://ieeexplore.ieee.org/abstract/document/4577833}{Secure
  Control: Towards Survivable Cyber-Physical Systems},'' \emph{2008 The 28th
  International Conference on Distributed Computing Systems Workshops}, pp.
  495--500, 2008.

\bibitem{trustandRobotsSycara}
R.~Liu, F.~Jia, W.~Luo, M.~Chandarana, C.~Nam, M.~Lewis, and K.~Sycara,
  ``\href{https://dl.acm.org/doi/pdf/10.5555/3306127.3331683?download=true}{Trust-Aware
  Behavior Reflection for Robot Swarm Self-Healing},'' \emph{Proceedings of the
  18th International Conference on Autonomous Agents and MultiAgent Systems},
  p. 122–130, 2019.

\bibitem{trustEstimationRobots}
K.~T. {Mantel} and C.~M. {Clark},
  ``\href{https://ieeexplore.ieee.org/stamp/stamp.jsp?tp=&arnumber=6491281}{Trust
  networks in multi-robot communities},'' \emph{2012 IEEE International
  Conference on Robotics and Biomimetics (ROBIO)}, pp. 2114--2119, Dec 2012.

\bibitem{robotTrustSchwager}
A.~Pierson and M.~Schwager,
  ``\href{http://citeseerx.ist.psu.edu/viewdoc/download?doi=10.1.1.711.6415&rep=rep1&type=pdf}{Adaptive
  Inter-Robot Trust for Robust Multi-Robot Sensor Coverage},'' in \emph{In
  International Symposium on Robotics Research}, 2013.

\bibitem{spoofResilientCoordinationusingFingerprints}
V.~{Renganathan} and T.~{Summers},
  ``\href{https://ieeexplore.ieee.org/stamp/stamp.jsp?tp=&arnumber=8250942}{Spoof
  resilient coordination for distributed multi-robot systems},'' \emph{2017
  International Symposium on Multi-Robot and Multi-Agent Systems (MRS)}, pp.
  135--141, Dec 2017.

\bibitem{GilLCSS}
S.~{Gil}, C.~{Baykal}, and D.~{Rus},
  ``\href{https://ieeexplore.ieee.org/stamp/stamp.jsp?tp=&arnumber=8405579}{Resilient
  Multi-Agent Consensus Using Wi-Fi Signals},'' \emph{IEEE Control Systems
  Letters}, vol.~3, no.~1, pp. 126--131, Jan 2019.

\bibitem{doi:10.1177/0278364914567793}
S.~Gil, S.~Kumar, D.~Katabi, and D.~Rus,
  ``\href{https://journals.sagepub.com/doi/pdf/10.1177/0278364914567793}{Adaptive
  communication in multi-robot systems using directionality of signal
  strength},'' \emph{The International Journal of Robotics Research}, vol.~34,
  no.~7, pp. 946--968, 2015.

\bibitem{1307346}
J.~{Newsome}, E.~{Shi}, D.~{Song}, and A.~{Perrig},
  ``\href{https://ieeexplore.ieee.org/stamp/stamp.jsp?tp=&arnumber=1307346}{The
  Sybil attack in sensor networks: analysis defenses},'' \emph{Third
  International Symposium on Information Processing in Sensor Networks, 2004.
  IPSN 2004}, pp. 259--268, April 2004.

\bibitem{Acemoglu2011}
D.~Acemoglu and A.~Ozdaglar,
  ``\href{https://dspace.mit.edu/bitstream/handle/1721.1/71547/Acemoglu10-15.pdf?sequence=1&isAllowed=y}{Opinion
  Dynamics and Learning in Social Networks},'' \emph{Dynamic Games and
  Applications}, vol.~1, no.~1, March 2011.

\bibitem{10.1145/2500423.2500444}
J.~Xiong and K.~Jamieson,
  ``\href{https://dl.acm.org/doi/pdf/10.1145/2500423.2500444?download=true}{SecureArray:
  Improving Wifi Security with Fine-Grained Physical-Layer Information},''
  \emph{Proceedings of the 19th Annual International Conference on Mobile
  Computing \& Networking}, p. 441–452, 2013.

\bibitem{6870484}
M.~{Yazdanian} and A.~{Mehrizi-Sani},
  ``\href{https://ieeexplore.ieee.org/stamp/stamp.jsp?tp=&arnumber=6870484}{Distributed
  Control Techniques in Microgrids},'' \emph{IEEE Transactions on Smart Grid},
  vol.~5, no.~6, pp. 2901--2909, Nov 2014.

\bibitem{7820067}
L.~{Guerrero-Bonilla}, A.~{Prorok}, and V.~{Kumar},
  ``\href{https://ieeexplore.ieee.org/stamp/stamp.jsp?arnumber=7820067}{Formations
  for Resilient Robot Teams},'' \emph{IEEE Robotics and Automation Letters},
  vol.~2, no.~2, pp. 841--848, April 2017.

\bibitem{10.1145/2185505.2185507}
H.~J. LeBlanc, H.~Zhang, S.~Sundaram, and X.~Koutsoukos,
  ``\href{https://dl.acm.org/doi/pdf/10.1145/2185505.2185507?download=true}{Consensus
  of Multi-Agent Networks in the Presence of Adversaries Using Only Local
  Information},'' \emph{Proceedings of the 1st International Conference on High
  Confidence Networked Systems}, p. 1–10, 2012.

\bibitem{1272911}
H.~G. {Tanner}, A.~{Jadbabaie}, and G.~J. {Pappas},
  ``\href{https://ieeexplore.ieee.org/stamp/stamp.jsp?tp=&arnumber=1272911}{Stable
  flocking of mobile agents part I: dynamic topology},'' \emph{42nd IEEE
  International Conference on Decision and Control (IEEE Cat. No.03CH37475)},
  vol.~2, pp. 2016--2021, Dec 2003.

\bibitem{6491281}
K.~T. {Mantel} and C.~M. {Clark},
  ``\href{https://ieeexplore.ieee.org/stamp/stamp.jsp?tp=&arnumber=6491281}{Trust
  networks in multi-robot communities},'' \emph{2012 IEEE International
  Conference on Robotics and Biomimetics (ROBIO)}, pp. 2114--2119, Dec 2012.

\bibitem{6315661}
H.~{Zhang} and S.~{Sundaram},
  ``\href{https://ieeexplore.ieee.org/stamp/stamp.jsp?tp=&arnumber=6315661}{Robustness
  of information diffusion algorithms to locally bounded adversaries},''
  \emph{2012 American Control Conference (ACC)}, pp. 5855--5861, June 2012.

\bibitem{6032057}
O.~{Kosut}, L.~{Jia}, R.~J. {Thomas}, and L.~{Tong},
  ``\href{https://ieeexplore.ieee.org/stamp/stamp.jsp?tp=&arnumber=6032057}{Malicious
  Data Attacks on the Smart Grid},'' \emph{IEEE Transactions on Smart Grid},
  vol.~2, no.~4, pp. 645--658, Dec 2011.

\bibitem{5399524}
F.~{Pasqualetti}, A.~{Bicchi}, and F.~{Bullo},
  ``\href{https://ieeexplore.ieee.org/stamp/stamp.jsp?tp=&arnumber=5399524}{On
  the security of linear consensus networks},'' \emph{Proceedings of the 48h
  IEEE Conference on Decision and Control (CDC) held jointly with 2009 28th
  Chinese Control Conference}, pp. 4894--4901, Dec 2009.

\bibitem{8250942}
V.~{Renganathan} and T.~{Summers},
  ``\href{https://ieeexplore.ieee.org/stamp/stamp.jsp?tp=&arnumber=8250942}{Spoof
  resilient coordination for distributed multi-robot systems},'' \emph{2017
  International Symposium on Multi-Robot and Multi-Agent Systems (MRS)}, pp.
  135--141, Dec 2017.

\bibitem{4543196}
A.~{Fagiolini}, M.~{Pellinacci}, G.~{Valenti}, G.~{Dini}, and A.~{Bicchi},
  ``\href{https://ieeexplore.ieee.org/stamp/stamp.jsp?tp=&arnumber=4543196}{Consensus-based
  Distributed Intrusion Detection for Multi-Robot Systems},'' \emph{2008 IEEE
  International Conference on Robotics and Automation}, pp. 120--127, May 2008.

\bibitem{6481629}
H.~J. {LeBlanc}, H.~{Zhang}, X.~{Koutsoukos}, and S.~{Sundaram},
  ``\href{https://ieeexplore.ieee.org/stamp/stamp.jsp?tp=&arnumber=6481629}{Resilient
  Asymptotic Consensus in Robust Networks},'' \emph{IEEE Journal on Selected
  Areas in Communications}, vol.~31, no.~4, pp. 766--781, April 2013.

\bibitem{Pierson13adaptiveinter-robot}
A.~Pierson and M.~Schwager,
  ``\href{http://citeseerx.ist.psu.edu/viewdoc/download?doi=10.1.1.711.6415&rep=rep1&type=pdf}{Adaptive
  Inter-Robot Trust for Robust Multi-Robot Sensor Coverage},'' in \emph{In
  International Symposium on Robotics Research}, 2013.

\bibitem{booker2018effects}
M.~Booker,
  ``\href{https://kb.osu.edu/bitstream/handle/1811/84561/MeghanBooker_Honors_Undergraduate_Thesis.pdf?sequence=1}{Effects
  of Hacking an Unmanned Aerial Vehicle Connected to the Cloud},'' Ph.D.
  dissertation, The Ohio State University, 2018.

\bibitem{origByz}
D.~Dolev,
  ``\href{https://www.cs.huji.ac.il/~dolev/pubs/byz-strike-again.pdf}{The
  Byzantine generals strike again},'' \emph{Journal of Algorithms}, vol.~3,
  no.~1, pp. 14 -- 30, 1982.

\bibitem{douceur2002sybil}
J.~R. Douceur,
  ``\href{https://www.microsoft.com/en-us/research/wp-content/uploads/2002/01/IPTPS2002.pdf}{The
  sybil attack},'' \emph{International workshop on peer-to-peer systems}, pp.
  251--260, 2002.

\bibitem{lamport}
L.~Lamport,
  ``\href{https://link.springer.com/article/10.1007/s00446-006-0005-x}{Fast
  Paxos},'' \emph{Distributed Computing}, vol.~19, no.~2, pp. 79--103, Oct
  2006.

\bibitem{fitch2012synthetic}
J.~P. Fitch, \emph{Synthetic aperture radar}.\hskip 1em plus 0.5em minus
  0.4em\relax Springer Science \& Business Media, 2012.

\bibitem{4132698}
H.~{Klausing},
  ``\href{https://ieeexplore.ieee.org/stamp/stamp.jsp?tp=&arnumber=4132698}{Feasibility
  of a Synthetic Aperture Radar with Rotating Antennas (ROSAR)},'' \emph{1989
  19th European Microwave Conference}, pp. 287--299, Sep. 1989.

\bibitem{10.1145/2639108.2639142}
S.~Kumar, S.~Gil, D.~Katabi, and D.~Rus,
  ``\href{https://dl.acm.org/doi/pdf/10.1145/2639108.2639142?download=true}{Accurate
  Indoor Localization with Zero Start-up Cost},'' \emph{Proceedings of the 20th
  Annual International Conference on Mobile Computing and Networking}, p.
  483–494, 2014.

\bibitem{hegselmann2002opinion}
R.~Hegselmann, U.~Krause \emph{et~al.},
  ``\href{https://www.math.fsu.edu/~dgalvis/journalclub/papers/02_05_2017.pdf}{Opinion
  dynamics and bounded confidence models, analysis, and simulation},''
  \emph{Journal of artificial societies and social simulation}, vol.~5, no.~3,
  2002.

\bibitem{8424838}
S.~{Chung}, A.~A. {Paranjape}, P.~{Dames}, S.~{Shen}, and V.~{Kumar},
  ``\href{https://ieeexplore.ieee.org/stamp/stamp.jsp?tp=&arnumber=8424838}{A
  Survey on Aerial Swarm Robotics},'' \emph{IEEE Transactions on Robotics},
  vol.~34, no.~4, pp. 837--855, Aug 2018.

\bibitem{8991431}
S.~Green and P.~Månsson,
  ``\href{http://lup.lub.lu.se/luur/download?func=downloadFile&recordOId=8991431&fileOId=8991432}{Autonomous
  control of unmanned aerial multi-agent networks in confined spaces},'' 2019.

\bibitem{1205192}
A.~{Jadbabaie}, {Jie Lin}, and A.~S. {Morse},
  ``\href{https://ieeexplore.ieee.org/stamp/stamp.jsp?tp=&arnumber=1205192}{Coordination
  of groups of mobile autonomous agents using nearest neighbor rules},''
  \emph{IEEE Transactions on Automatic Control}, vol.~48, no.~6, pp. 988--1001,
  June 2003.

\bibitem{1605401}
R.~{Olfati-Saber},
  ``\href{https://ieeexplore.ieee.org/stamp/stamp.jsp?tp=&arnumber=1605401}{Flocking
  for multi-agent dynamic systems: algorithms and theory},'' \emph{IEEE
  Transactions on Automatic Control}, vol.~51, no.~3, pp. 401--420, March 2006.

\bibitem{10.1145/37401.37406}
C.~W. Reynolds,
  ``\href{https://dl.acm.org/doi/pdf/10.1145/37401.37406?download=true}{Flocks,
  Herds and Schools: A Distributed Behavioral Model},'' \emph{Proceedings of
  the 14th Annual Conference on Computer Graphics and Interactive Techniques},
  p. 25–34, 1987.

\bibitem{BulloCyberphysSecurity_GeometricPrinciples}
F.~{Pasqualetti}, F.~{Dorfler}, and F.~{Bullo},
  ``\href{https://ieeexplore.ieee.org/abstract/document/7011011}{Control-Theoretic
  Methods for Cyberphysical Security: Geometric Principles for Optimal
  Cross-Layer Resilient Control Systems},'' \emph{IEEE Control Systems
  Magazine}, vol.~35, no.~1, pp. 110--127, 2015.

\bibitem{obs_estimation}
S.~Sundaram and C.~N. Hadjicostis,
  ``\href{https://ieeexplore.ieee.org/document/4497790}{Distributed function
  calculation and consensus using linear iterative strategies},'' \emph{IEEE
  journal on selected areas in communications}, vol.~26, no.~4, pp. 650--660,
  2008.

\bibitem{su2007flocking1}
H.~Su, X.~Wang, and Z.~Lin,
  ``\href{https://ieeexplore.ieee.org/stamp/stamp.jsp?arnumber=4434066}{Flocking
  of multi-agents with a virtual leader part I: with a minority of informed
  agents},'' \emph{2007 46th IEEE Conference on Decision and Control}, pp.
  2937--2942, 2007.

\bibitem{su2007flocking2}
H.~Su, X.~Wang, and Z.~Lin,
  ``\href{https://ieeexplore.ieee.org/stamp/stamp.jsp?arnumber=4434067}{Flocking
  of multi-agents with a virtual leader part II: with a virtual leader of
  varying velocity},'' \emph{2007 46th IEEE Conference on Decision and
  Control}, pp. 1429--1434, 2007.

\bibitem{dubhashi2009concentration}
D.~P. Dubhashi and A.~Panconesi, \emph{Concentration of measure for the
  analysis of randomized algorithms}.\hskip 1em plus 0.5em minus 0.4em\relax
  Cambridge University Press, 2009.

\bibitem{ICALPlower}
I.~Diakonikolas, T.~Gouleakis, J.~Peebles, and E.~Price,
  ``\href{https://drops.dagstuhl.de/opus/volltexte/2018/9045/}{Sample-Optimal
  Identity Testing with High Probability},'' \emph{45th International
  Colloquium on Automata, Languages, and Programming (ICALP 2018)}, vol. 107,
  pp. 41:1--41:14, 2018.

\bibitem{HR90}
T.~Hagerup and C.~R\"{u}b,
  ``\href{https://reader.elsevier.com/reader/sd/pii/002001909090214I?token=B4B31579EE66C6F4EF6577FCB8F38CCF1C8456B131FF5FC98652AA7C1FB206578C75B8C4088FAD6F6C50A15DEC18AEBF}{A
  Guided Tour of Chernoff Bounds},'' \emph{Information Processing Letters},
  vol.~33, no.~6, pp. 305--308, 1990.

\end{thebibliography}

%
\begin{IEEEbiography}[{\includegraphics[width=1in,height=1.25in,clip,keepaspectratio]{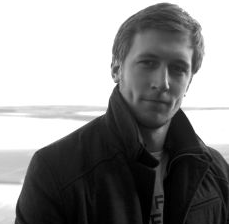}}]{Frederik Mallmann-Trenn}
is an assistant professor (UK lecturer) working on stochastic processes. He received his PhD from \'Ecole normale su\'erieure where he focused on stochastic processes in the context of distributed computing.
\end{IEEEbiography}
\vspace{-0.1in}
\begin{IEEEbiography}[{\includegraphics[width=1in,height=1.25in,clip,keepaspectratio]{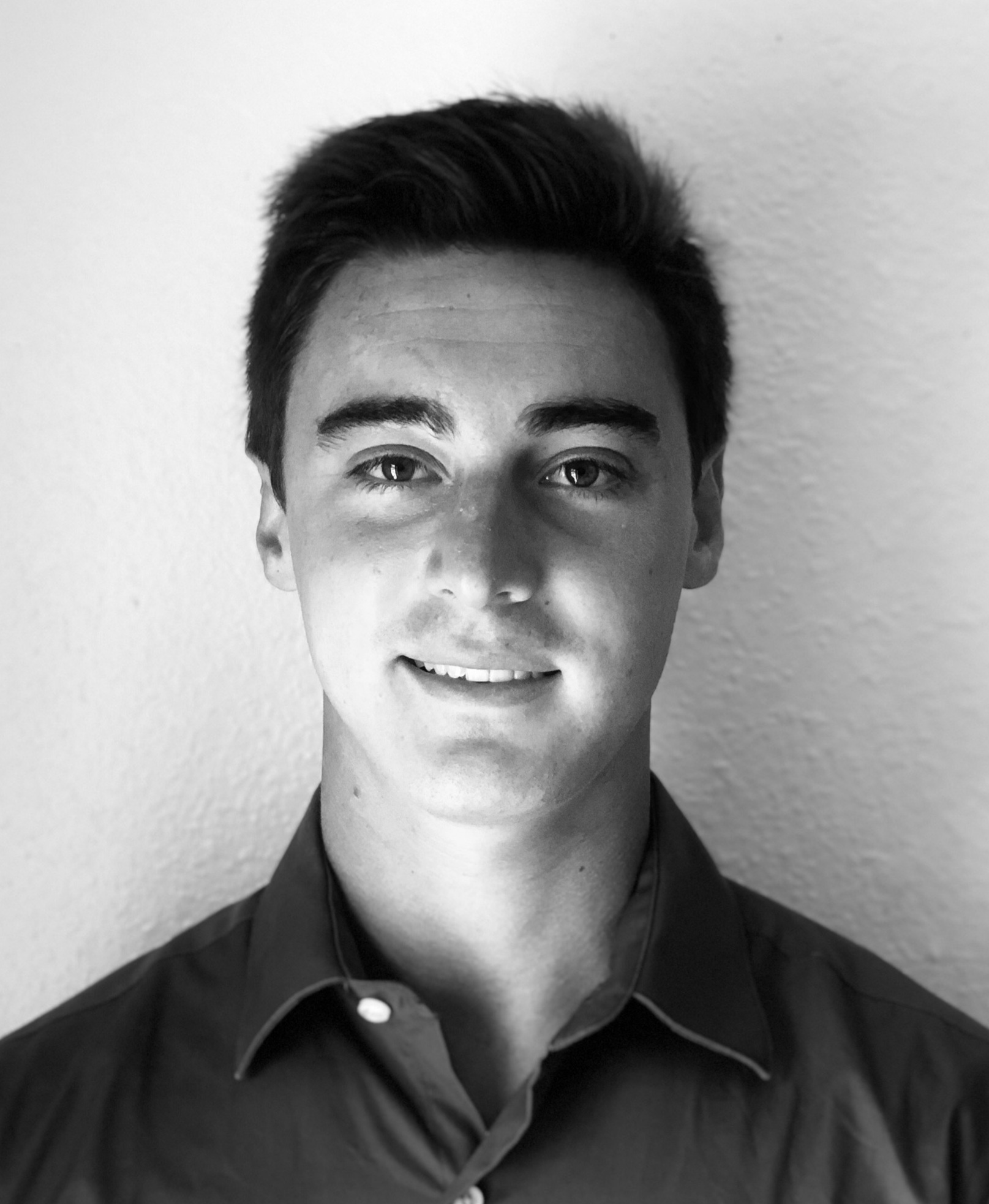}}]{Matthew Cavorsi}
is a Mechanical Engineering Ph.D. student in the School of Engineering and Applied Sciences at Harvard University, advised by Prof. Stephanie Gil. He works on multi-robot coordination and resilience to adversaries in multi-robot systems. He received his Bachelor’s degree in Aerospace Engineering from the Pennsylvania State University in 2017.
\end{IEEEbiography}
\vspace{-0.1in}
\begin{IEEEbiography}[{\includegraphics[width=1in,height=1.25in,clip,keepaspectratio]{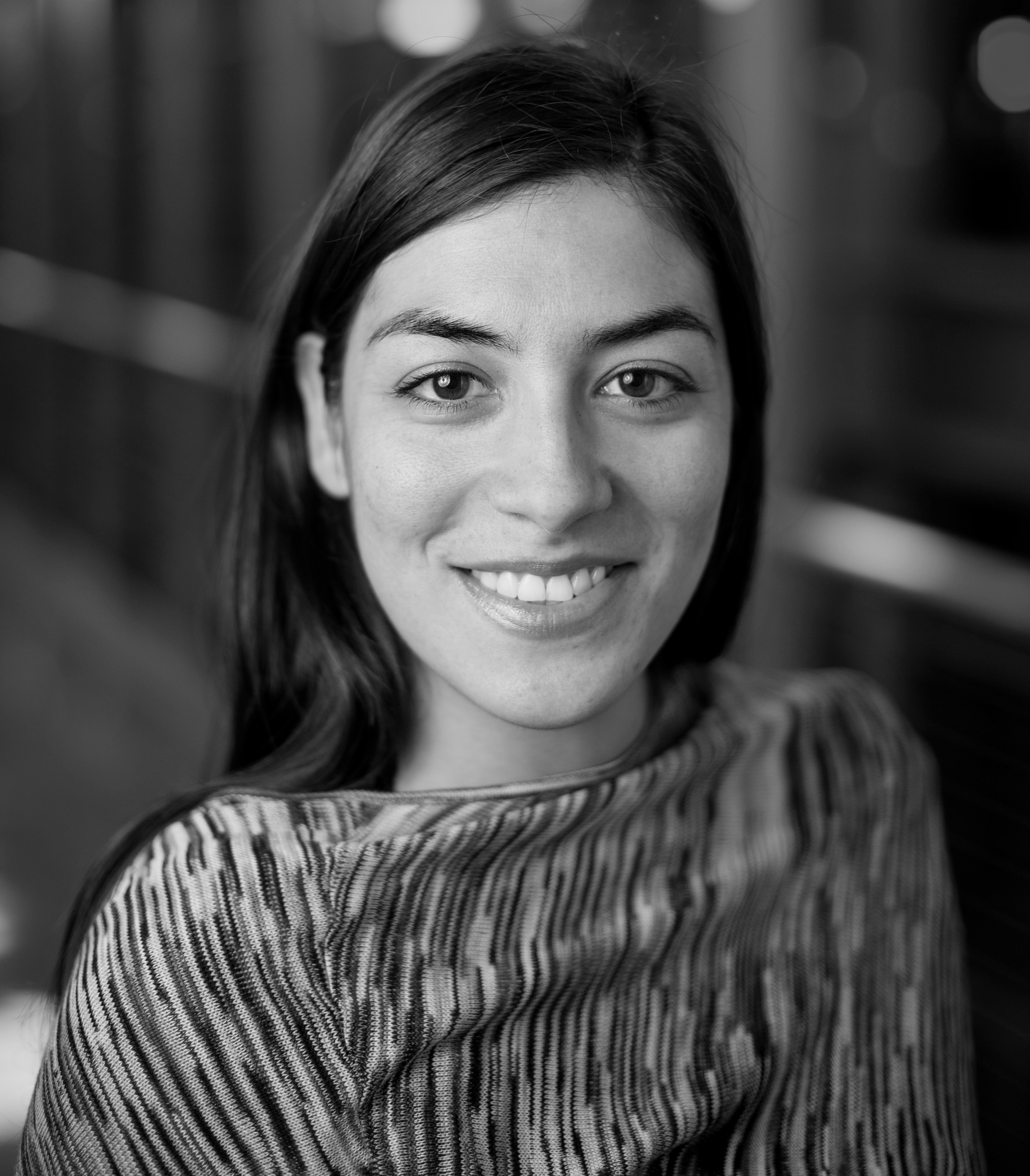}}]{Stephanie Gil}
is an Assistant Professor in the Computer Science Department at the School of Engineering and Applied Sciences at Harvard University where she directs the Robotics, Embedded Autonomy and Communication Theory (REACT) Lab. Prior she was an Assistant Professor at Arizona State University.  Her research focuses on multi-robot systems where she studies the impact of information exchange and communication on resilience and trusted coordination. She is the recipient of the 2019 Faculty Early Career Development Program Award from the National Science Foundation (NSF CAREER), and has been selected as a 2020 Alfred P. Sloan Fellow.  She obtained her PhD from the Massachusetts Institute of Technology in 2014. 
\end{IEEEbiography}




\end{document}